\documentclass{article}

\usepackage{microtype}
\usepackage{graphicx}
\usepackage{booktabs} % for professional tables

\usepackage{hyperref}

\usepackage[accepted]{icml2025}

\usepackage{hyperref}
\usepackage{url}

\usepackage{amsmath}
\usepackage{amssymb}
\usepackage{mathtools}
\usepackage{amsthm}

\usepackage{microtype}
\usepackage{graphicx}
\usepackage{booktabs} % for professional tables

\usepackage{algorithm}
\usepackage{algorithmic}
\usepackage{graphicx} 
\usepackage{todonotes}
\usepackage{dsfont}
\usepackage{nicefrac}
\usepackage{tcolorbox}
\usepackage{tabularx}
\usepackage{pifont}
\usepackage{enumitem}
\usepackage{multirow}
\usepackage{caption} 
\usepackage{wrapfig}
\usepackage[export]{adjustbox}
\usepackage[list=true]{subcaption}
\usepackage{colortbl}
\definecolor{lightgray}{RGB}{242, 242, 242}
\usepackage{float}
\usepackage[capitalize,noabbrev]{cleveref}
\crefname{assumption}{Assumption}{Assumptions}
\crefname{algorithm}{Algorithm}{Algorithms}
\crefname{informalassumption}{Informal Assumption}{Informal Assumptions}

\usepackage{longtable}
\usepackage[title]{appendix}

\usepackage{tcolorbox}
\usepackage{xspace}
\DeclareRobustCommand{\ie}{i.e.,\@\xspace}
  
\DeclareRobustCommand{\eg}{e.g.,\@\xspace}
\DeclareRobustCommand{\wrt}{w.r.t.\@\xspace}
\usepackage[capitalize]{cleveref}		% for cleveref formatting

\usepackage{tikz}
\usepackage{tikz-cd}
\usetikzlibrary{shapes,backgrounds}
\usepackage{array}
\usepackage{enumitem}

\usepackage{wrapfig}
\usepackage[export]{adjustbox}
\usepackage[list=true]{subcaption}
\usepackage[normalem]{ulem}
\usepackage{rotating}   % for \rotatebox (rotates text)
\usepackage{cleveref}
\usepackage{adjustbox}
\usepackage{twoopt}

\usepackage{acronym}		% for acronyms
\newcommand{\acdef}[1]{\textit{\acl{#1}} \textup{(\acs{#1})}\acused{#1}}		% for acro def
\usepackage{amsthm}
\usepackage{thmtools, thm-restate}
\declaretheorem[numberwithin=section]{thm}
\declaretheorem[sibling=thm]{theorem}

\declaretheorem[numberwithin=section]{assumption}
\declaretheorem[]{proof sketch}
\declaretheorem[]{definition}

\usepackage{amsfonts}[mathscr]
\usepackage{amssymb}
\usepackage{amsmath}
\DeclareMathOperator*{\EV}{\mathbb{E}}

\newcommand{\A}{\mathcal{A}}

\newcommand{\supp}{\mathrm{supp}}
\DeclareMathOperator*{\argmax}{arg\,max}

\newcommand{\X}{\mathcal{X}}
\newcommand{\R}{\mathbb{R}}

\newcommand{\F}{\mathcal{F}}

\newcommand{\J}{\mathcal{J}}
\newcommand{\U}{\mathcal{U}}
\newcommand{\Q}{\mathcal{Q}}
\newcommand{\mypar}[1]{\textbf{#1.}}

\newcommand{\entropy}{\mathcal{H}}
\DeclareMathOperator{\mP}{\mathbb{P}}

\DeclareMathOperator{\mF}{\mathbb{F}}
\newcommand{\der}{\mathrm{d}}

\newcommand{\LinearFineTuningSolver}{\textsc{\small{LinearFineTuningSolver}}\xspace}
\newcommand{\noise}{U}
\newcommand{\bias}{b}
\newcommand{\hist}{\mathcal{G}}
\newcommand{\pMD}{p_\sharp}
\newcommand{\step}{\gamma}

\newcommand{\norm}[1]{\left\| #1 \right\|}

\newcommand{\AlgNameLong}{Score-based Maximum Entropy Manifold Exploration\xspace}
\newcommand{\AlgNameShort}{\textsc{\small{S-MEME}}\xspace}
\newcommand{\AlgNameDef}{\textbf{S}core-based \textbf{M}aximum \textbf{E}ntropy \textbf{M}anifold \textbf{E}xploration (\textsc{\small{S-MEME}}\xspace)}

\definecolor{myviolet}{rgb}{0.6, 0.4, 0.8}
\definecolor{mygreen}{rgb}{0.0, 0.5, 0.0}

\DeclareMathOperator{\dist}{dist}		% for distance
\definecolor{pastelblueold}{RGB}{56,146,236}
\definecolor{pastelblue}{RGB}{43,115,187}
\definecolor{pastelgreen}{RGB}{63,159,95}

\hypersetup{
  colorlinks=true,
  citecolor=pastelgreen,
  linkcolor=pastelblue,    % internal references (equations, figures, tables)
  urlcolor=pastelblue      % URLs
}

\icmltitlerunning{Provable Maximum Entropy Manifold Exploration via Diffusion Models}

\begin{document}

\twocolumn[
\icmltitle{Provable Maximum Entropy Manifold Exploration via Diffusion Models}



\icmlsetsymbol{equal}{*}

\begin{icmlauthorlist}
\icmlauthor{Riccardo De Santi}{equal,eth,ethai}
\icmlauthor{Marin Vlastelica}{equal,eth,ethai}
\icmlauthor{Ya-Ping Hsieh}{eth}
\icmlauthor{Zebang Shen}{eth}
\icmlauthor{Niao He}{eth,ethai}
\icmlauthor{Andreas Krause}{eth,ethai}
\end{icmlauthorlist}

\icmlaffiliation{eth}{ETH Zurich, 8092 Zurich, Switzerland}
\icmlaffiliation{ethai}{ETH AI Center, Zurich, Switzerland}

\icmlcorrespondingauthor{Riccardo De Santi}{rdesanti@ethz.ch}

\icmlkeywords{Machine Learning, ICML}

\vskip 0.3in
]



\printAffiliationsAndNotice{\icmlEqualContribution} % otherwise use the standard text.

\begin{abstract}
\looseness -1 Exploration is critical for solving real-world decision-making problems such as scientific discovery, where the objective is to generate truly novel designs rather than mimic existing data distributions.  In this work, we address the challenge of leveraging the representational power of generative models for exploration without relying on explicit uncertainty quantification. We introduce a novel framework that casts exploration as entropy maximization over the approximate data manifold implicitly defined by a pre-trained diffusion model. Then, we present a novel principle for exploration based on density estimation, a problem well-known to be challenging in practice. To overcome this issue and render this method truly scalable, we leverage a fundamental connection between the entropy of the density induced by a diffusion model and its score function. Building on this, we develop an algorithm based on mirror descent that solves the exploration problem as sequential fine-tuning of a pre-trained diffusion model. We prove its convergence to the optimal exploratory diffusion model under realistic assumptions by leveraging recent understanding of mirror flows. Finally, we empirically evaluate our approach on both synthetic and high-dimensional text-to-image diffusion, demonstrating promising results.
\end{abstract}

\section{Introduction} 
\label{sec:introduction}
Recent progress in generative modeling, particularly the emergence of diffusion models~\citep{sohl2015deep, song2019generative, ho2020denoising}, has achieved unprecedented success in generating high-quality samples across diverse domains, including chemistry~\citep{hoogeboom2022equivariant}, biology~\citep{corso2022diffdock}, and robotics~\citep{chi2023diffusion}. Traditionally, generative models have been employed to capture the underlying data distribution in high-dimensional spaces, facilitating processes such as molecule generation or material synthesis~\citep{bilodeau2022generative, zeni2023mattergen}.  However, simply approximating the data distribution is insufficient for real-world discovery, where exploration beyond high (data) density regions is essential. 

\begin{figure}[t]
    \centering
    \includegraphics[width=0.33\textwidth]{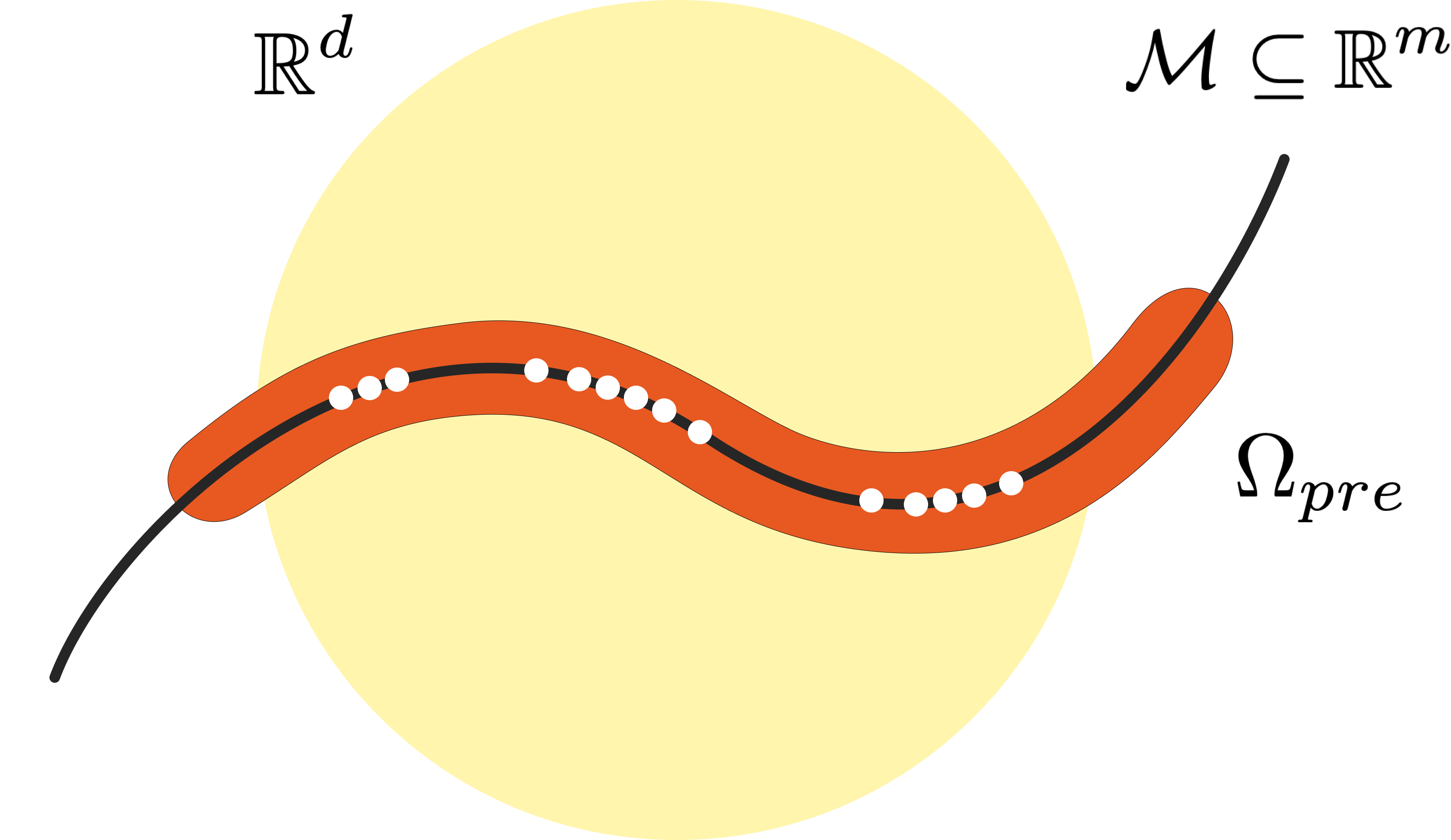}
    \caption{\looseness -1 A diffusion model $\pi^{pre}$ pre-trained on a set of points (white) implicitly learns a set  $\Omega_{pre}$ (orange) approximating the true low-dimensional data manifold $\mathcal{M} \subseteq \R^m$ (black) with $m \ll d $. The approximate data manifold $\Omega_{pre}$ can be significantly smaller than $\R^d$ (yellow).} \vspace{-0mm}
    \label{fig:low_dim_manifold}
\end{figure}
 
Nonetheless, as illustrated in Figure \ref{fig:low_dim_manifold}, these models excel at capturing complex data manifolds that are often significantly lower-dimensional than the ambient space~\citep{stanczukdiffusion, kamkari2024geometric, chen2023score}, and can synthesize realistic novel samples that satisfy intricate constraints (\eg valid drug molecules or materials). Yet, when the goal shifts to exploring novel regions within that manifold, a fundamental question remains: 
\vspace{-0mm}
\begin{center} 
\emph{How can we leverage the representational power of generative models to guide exploration?} \vspace{-0mm}
\end{center}
\paragraph{{Our approach}}
In this work, we tackle this challenge by first introducing the \emph{maximum entropy manifold exploration} problem (\cref{sec:problem_setting}). This involves learning a continuous-time reinforcement learning policy~\citep{doya2000reinforcement, zhao2024scores} that governs a new diffusion model to optimally explore the approximate data manifold implicitly captured by a pre-trained model. To this end, we present a theoretically grounded algorithmic principle that enables self-guided exploration via a diffusion model's own representational power of the density it induces (\cref{sec:principle}). This turns exploration into density estimation, a task well-known to be challenging in high-dimensional real-world settings~\citep{song2020score,kingma2021variational, skreta2024superposition}. To overcome this obstacle and render the method proposed truly scalable, we leverage a fundamental connection between the entropy of the density induced by a diffusion model and its score function.
Building on this, we propose a practical algorithm that performs manifold exploration through sequential fine-tuning of the pre-trained model (\cref{sec:algorithm}). We provide theoretical convergence guarantees for optimal exploration in a simplified illustrative setting by interpreting the algorithm proposed as a mirror descent scheme~\citep{nemirovskij1983problem, lu2018relatively} (\cref{sec:theory1}), and then generalize the analysis to realistic settings building on recent understanding of mirror flows~\citep{hsieh2019finding} (\cref{sec:theory2}). Finally, we provide an experimental evaluation of the proposed method, demonstrating its practical relevance on both synthetic and high-dimensional image data, where we leverage a pre-trained text-to-image diffusion model (\cref{sec:experiments}). \vspace{-0mm}
\paragraph{{Our contributions}} To sum up, in this work we present the following contributions:
\begin{itemize}[noitemsep,topsep=0pt,parsep=0pt,partopsep=0pt,leftmargin=*]
    \item The maximum entropy manifold exploration problem, that captures the goal of exploration over the approximate data manifold implicitly represented by a pre-trained diffusion model (\cref{sec:problem_setting})
    \item A scalable algorithmic principle for manifold exploration that leverages the representational power of a pre-trained diffusion model (\cref{sec:principle}), and a theoretically grounded algorithm based on sequential fine-tuning (\cref{sec:algorithm}).
    \item Convergence guarantees for the algorithm presented both under simplified and realistic assumptions leveraging recent understanding of mirror flows (Sections \ref{sec:theory1} and \ref{sec:theory2}).
    \item An experimental evaluation of the proposed method showcasing its practical relevance on both an illustrative task and a high-dimensional text-to-image setting (\cref{sec:experiments}).
\end{itemize}

\begin{figure*}[htbp]
    \centering
    \begin{subfigure}{0.64\textwidth} % 2/3 width
        \centering
        \includegraphics[width=\textwidth, height=0.34\textwidth, keepaspectratio]{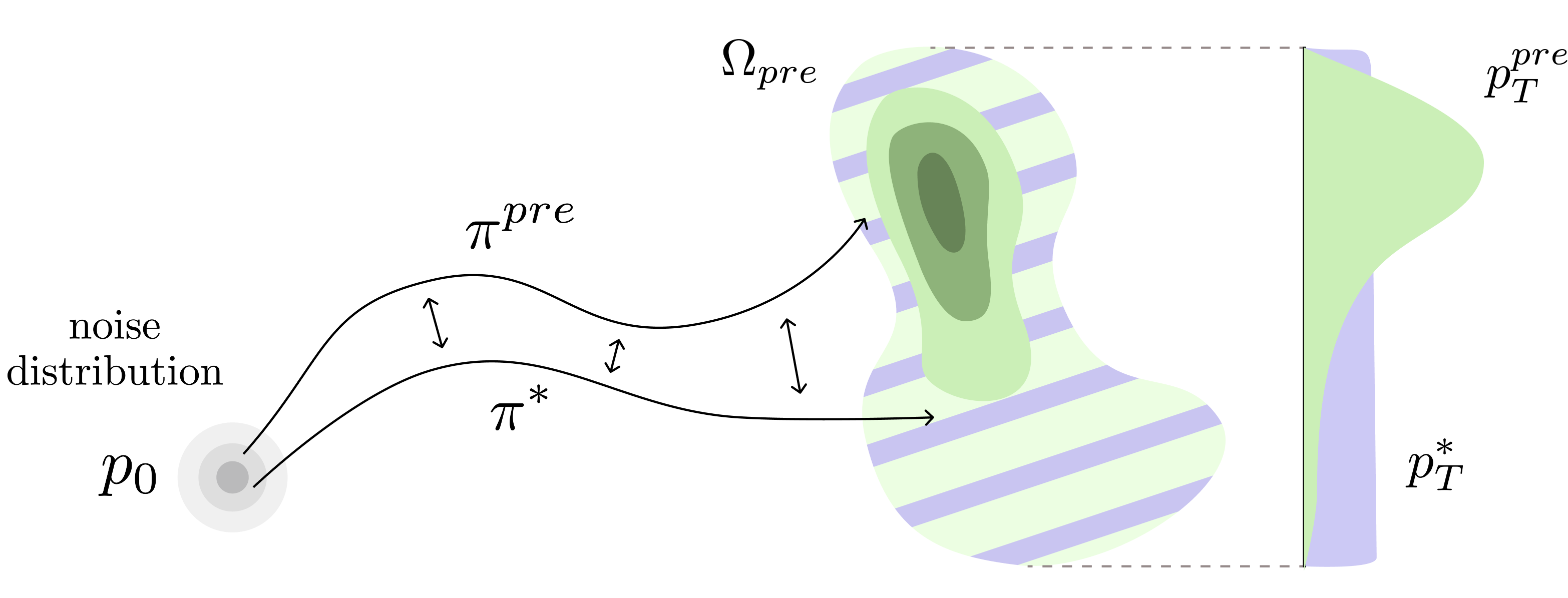}
        \caption{Pre-trained and fine-tuned diffusion model processes.}
        \label{fig:problem_processes}
    \end{subfigure}%
    \hspace{17pt}
    \begin{subfigure}{0.32\textwidth} % 1/3 width
        \centering
        \includegraphics[width=\textwidth, height=0.34\textwidth, keepaspectratio]{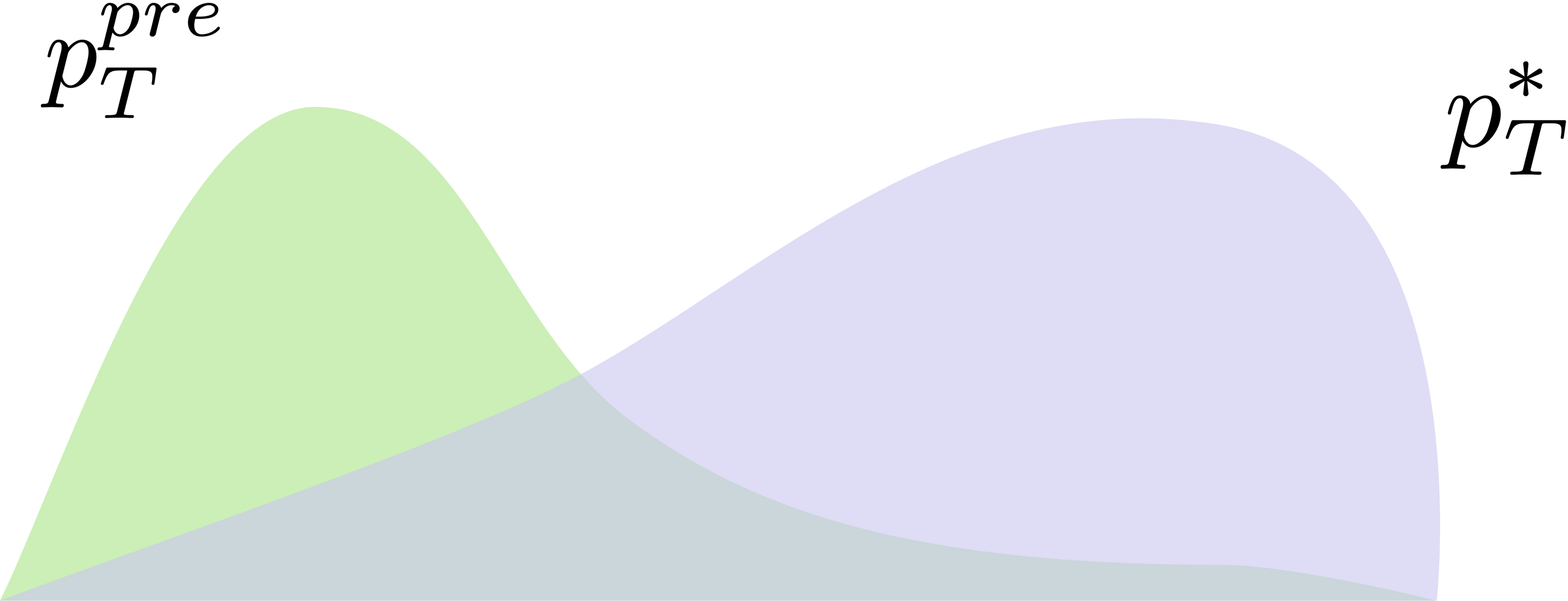}
        \caption{Regularized surprise maximization}
        \label{fig:surprise_maximization}
    \end{subfigure}\vspace{-0pt}
    \caption{\looseness -1 (\ref{fig:problem_processes}) Diffusion processes and marginal densities corresponding to the pre-trained model $\pi^{pre}$ (green), and maximally explorative fine-tuned model $\pi^*$ (violet). (\ref{fig:surprise_maximization}) Fine-tuning a pre-trained diffusion model (green) via the surprise maximization principle in Eq. \eqref{eq:entropy_first_variation} one obtains a diffusion model (violet) able to sample low-density regions.}
    \label{fig:top_figure} \vspace{-0pt}
\end{figure*}

\section{Background and Notation}
\label{sec:background}
\mypar{General Notation}
We denote with $\X \subseteq \R^d$ an arbitrary set. Then, we indicate the set of Borel probability measures on $\X$ with $\mP(\X)$, and the set of functionals over the set of probability measures $\mP(\X)$ as $\mF(\X)$. We write $\der \mu = \rho \der x$ to express that the density function of $\mu \in \mP(\X)$ with respect to the Lebesgue measure is $\rho$. Along this work, all integrals without an explicit measure are interpreted \wrt the Lebesgue measure. Given an integer $N$, we define $[N] \coloneqq \{1, \ldots, N\}$. Moreover, for two densities $\mu, \nu \in \mP(\X)$, we denote with $D_{KL}(\mu, \nu)$ the forward Kullback–Leibler divergence between $\mu$ and $\nu$. Ultimately, we denote by $\U[0,a]$ the uniform density over the bounded set $[0,a]$ with $a \in \R_+$.

\mypar{Continuous-time diffusion models}
Diffusion models (DMs) are deep generative models that approximately sample a complex data distribution by learning from observations a dynamical system to map noise to novel valid data points~\citep{song2019generative}. First, we introduce a forward stochastic differential equation (SDE) transforming to noise data points sampled from the data distribution $p_{data}$ :
\begin{equation}
    \der X_t = f(X_t, t) \der t + g(t) \der B_t \text{ with } X_0 \sim p_{data}  \label{eq:forward_process}
\end{equation}
where $X_t \in \R^d$ represents a $d$-dimensional point, $(B_t, t \geq 0)$ is $d$-dimensional Brownian motion, $f: \R_+ \times \R^d \to \R^d$ is a drift coefficient, and $g: \R_+ \to  \R_+$ is a diffusion coefficient. We denote with $p_t$ the marginal density at time $t$. Given a time horizon $t > 0$, one can sample $X_t \sim p_t$ by running the forward SDE in Equation \eqref{eq:forward_process}. We denote the time-reversal process by $X^{rev}_t \coloneqq X_{T-t}$ for $0 \leq t \leq T$, following the backward SDE:
\begin{equation}
    \der X^{rev}_t = f^{rev}(X^{rev}_t, T-t) \der t + \eta g(T-t) \der B_t \label{eq:notation_backward_SDE}
\end{equation}
with $f^{rev}(X^{rev}_t, T-t)$ corresponding to:
\begin{equation*}
    -f(X^{rev}_t, T-t) + \frac{1 + \eta^2}{2}g^2(T-t)\nabla_x \log p_{T-t}(X^{rev})
\end{equation*}
where $\nabla_x \log p_t(x)$ is the score function, and $\eta \in [0,1]$. By following the backward SDE in Equation \eqref{eq:notation_backward_SDE} from $X_T \sim p_T$, after $T$ steps one obtains $X_0 = X^{rev}_T \sim p_0 = p_{data}$. In practice, $p_T$ is replaced by a fixed and data-independent noise distribution $p_\infty \approx p_T$ for large $T$, typically a Gaussian\footnote{In the following, we will choose $p_\infty$ to be a truncated Gaussian for the sake of theoretical analysis.}, motivated by an asymptotic analysis of certain diffusion dynamics \citep{tang2024contractive}.

\mypar{Score matching and generation}
Since the score function $\nabla_x \log p_t(x)$ is unknown, it is typically approximated by a neural network $s_\theta(x,t)$ learned by minimizing the MSE at points sampled according to the forward process, namely:
\begin{equation*}
    \J(\theta) \coloneqq \EV_{t \sim \U[0,T]} \EV_{x \sim p_t} \left[\omega(t) \| s_\theta(x,t) - \nabla_x \log p_t(x) \|_2^2 \right] 
\end{equation*}
\looseness -1 where $\omega: [0,T] \to \R_{>0}$ is a weighting function. Crucially, this is equivalent to the denoising score matching objective \citep{vincent2011connection} consisting in estimating a minimizer $\theta^*$ of:
\begin{equation*}
    \EV_{t \sim \U[0,T]} \omega(t) \EV_{x_0 \sim p_0} \EV_{x_t | x_0} \| s_\theta(x_t,t) - \nabla_{x_t} \log p_t(x_t | x_0) \|_2^2 ]
\end{equation*}
\looseness -1 where $p_t(\cdot \mid x_0)$ is the conditional distribution of $x_t$ given an initial sample $x_0 \sim p_0$, which has a closed-form for typical diffusion dynamics. Once an approximate score function $s_{\theta^*}$ is learned, it can generate novel points approximately sampled from the data distribution. This is achieved by sampling an initial noise sample $X_0^{\leftarrow} \sim p_\infty$ and following the backward SDE in Equation \eqref{eq:notation_backward_SDE}, replacing the true score $\nabla_x \log p_t(x)$ with $s_{\theta^*}$, leading to the process $\{ X_t^{\leftarrow}\}_{t\in [0,T]}$. 
Next, we introduce a framework that we will leverage to fine-tune a pre-trained diffusion model.
    
\mypar{Continuous-time reinforcement learning}
We formulate finite-horizon continuous-time reinforcement learning (RL) as a specific class of stochastic control problems \citep{wang2020reinforcement, jia2022policy, zhao2024scores}. Given a state space $\X$ and an action space $\A$, we consider the transition dynamics governed by the following diffusion process, where we invert the direction of the time variable:
\begin{equation} \label{eqn_continuous_RL}
    \der \overline{X}_t = b(\overline{X}_t, t, a_t) \der t + \sigma(t) \der B_t \text{  with  } \overline{X}_0 \sim \mu
\end{equation}
\looseness -1 where $\mu \in \mP(\X)$ is an initial state distribution, $(B_t, t \geq 0)$ is $d$-dimensional Brownian motion, $b: \X \times \A \to \R^d$ is the drift coefficient, $\sigma: [0,T] \to \R_+$ is the diffusion coefficient, and $a_t \in \A$ is a selected action. In the following, we consider a state space $\X \coloneqq \R^d \times [0,T]$, and denote by (Markovian) policy a function $\pi(X_t, t) \in \mP(\A)$ mapping a state $(x,t) \in \X$ to a density over the action space $\A$, and denote with $p_t^\pi$ the marginal density at time $t$ induced by policy $\pi$. In particular, we will consider deterministic policies so that $a_t = \pi(X_t, t)$.  

\mypar{Pre-trained diffusion model as an RL policy}
A pre-trained diffusion model with score function $s^{pre}$ can be interpreted as an action process $a_t^{pre} \coloneqq s^{pre}(X_t^{\leftarrow}, T-t)$, where $a_t^{pre}$ is sampled from a continuous-time RL policy $a_t^{pre} \sim \pi^{pre}$. As a consequence, we can express the backward SDE induced by the pre-trained score $s^{pre}$ as follows:
\begin{equation} 
    \der X_t^{\leftarrow} = b(X_t^{\leftarrow}, t, a_t^{pre}) \der t + \eta \sigma (t) \der B_t \label{eq:generative_SDE}
\end{equation}
where we define $b(x, t, a) \coloneqq -f(x, T-t) + \frac{1+\eta^2}{2} g^2(T-t) \cdot a$ and $\sigma (t) = g(T-t)$ \citep{zhao2024scores}. In the following, we denote the pre-trained diffusion model by its (implicit) policy $\pi^{pre}$, which induces a marginal density $p^{pre}_T \coloneqq p^{\pi^{pre}}_T$ approximating the data distribution $p_{data}$.

\section{Problem Setting: Maximum Entropy Manifold Exploration}
\label{sec:problem_setting}
In this work, we aim to fine-tune a pre-trained diffusion model $\pi^{pre}$ to obtain a new model $\pi^*$, inducing a process:
\begin{equation}
    \der \overline{X}_t = b(\overline{X}_t , t, a_t^*) \der t + \eta \sigma (t) \der B_t \text{ with } a^*_t \sim \pi^*_t \label{eq:RL_diffusion_process}
\end{equation}
that rather than imitating the data distribution $p_{data}$ aims to induce a marginal state distribution $p^{\pi^*}_T$ that maximally explores the \emph{approximate data manifold} $\Omega_{pre}$ defined as:
\begin{equation}
    \Omega_{pre} = \supp (p_T^{pre}) \label{eq:support_def}
\end{equation}
\looseness -1 Formally, we pose the exploration problem as optimization of an entropy functional over the space of marginal distributions $p_T^\pi$ supported over the approximate data manifold, as shown in Figure \ref{fig:problem_processes}. Crucially, $\Omega_{pre}$, which is typically a complex set, \eg a molecular space, is defined only implicitly via the pre-trained diffusion model $\pi^{pre}$ as expressed in Equation \eqref{eq:support_def}. Formally, we state the exploration problem as follows.
\begin{tcolorbox}[colframe=white!, top=2pt,left=2pt,right=2pt,bottom=2pt]
\begin{center}
\textbf{Maximum Entropy Manifold Exploration}
\begin{align}
    &\argmax_{\pi}\quad  \entropy \left( p^{\pi}_T \right) \label{eq:opt_problem}\\ 
    &s.t. \quad p^{\pi}_T \in \mP(\Omega_{pre}) \notag
\end{align}
\end{center}
\end{tcolorbox}
In this formulation, $\entropy \in \mF(\Omega_{pre})$ denotes the differential entropy functional quantifying exploration, expressed as: 
\begin{equation}
 \entropy(\mu) = - \int \der \mu \log \frac{\der \mu}{dx}, \quad \mu \in \mP(\Omega_{pre}) \label{eq:entropy_functional_def}
\end{equation}
For this objective to be well defined, \ie the maximum is achieved by some measure $\mu \in \mP(\Omega_{pre})$, a sufficient condition is stated in the following and proved in Appendix \ref{sec:compactness}.

\begin{restatable}[$\Omega_{pre}$ is compact]{proposition}{omegaIsCompact}
\label{proposition:omega_is_compact}
\looseness -1 Suppose that $s^{pre}$ is Lipschitz and the noise distribution $p_0$ is chosen as the truncated Gaussian. Then $\Omega_{pre}$ spanned by an ODE sampler is compact.
\end{restatable}

\looseness -1 Notice that the assumptions in Proposition \ref{proposition:omega_is_compact} are standard for analysis of diffusion processes, \citep[\eg][]{lee2022convergence, pidstrigach2022score}, and not limiting in practice. Proposition \ref{proposition:omega_is_compact} implies that $\Omega_{pre}$ is a bounded subset of $\R^d$, which according to the manifold hypothesis approximates a lower-dimensional data manifold ~\citep{li2024diffusion, chen2023score, stanczukdiffusion, kamkari2024geometric}, as illustrated in Figure \ref{fig:low_dim_manifold}.

Crucially, both the constraint set $\mP(\Omega_{pre})$ and the marginal density $p_T^\pi$ in Problem \ref{eq:opt_problem} are \emph{never represented explicitly}, but only implicitly as functions of the pre-trained policy $\pi^{pre}$ and of the new policy $\pi$ respectively. 

\looseness -1 In the rest of this work, we show that Problem \eqref{eq:opt_problem} can be solved by fine-tuning the initial pre-trained model with respect to rewards obtained by sequentially linearizing the entropy functional. Towards this goal, in the next section, we introduce a scalable algorithmic principle that guides exploration by leveraging the representational capacity of the pre-trained diffusion model.

\section{A Principle for Scalable Exploration}
\label{sec:principle}
\looseness -1 As a first step towards tackling the maximum entropy manifold exploration problem in Equation \eqref{eq:opt_problem}, we introduce a principle for exploration corresponding to a specific (intrinsic) reward function for fine-tuning. 

\looseness -1 To this end, we define the first variation of a functional over a space of probability measures~\cite{hsieh2019finding}. A functional $\F \in \mF(\X)$, where $\F: \mP(\X) \to \R$, has first variation at $\mu \in \mP(\X)$ if there exists a function $\delta \F(\mu) \in \mF(\X)$ such that for all $\mu' \in \mP(\X)$ it holds that:
\begin{equation*}
    \F(\mu + \epsilon \mu') = \F(\mu) + \epsilon \langle \mu', \delta \F(\mu) \rangle + o(\epsilon).
\end{equation*}
where the inner product is interpreted as an expectation.

We can now present the following exploration principle as KL-regularized fine-tuning of the pre-trained model to maximize the entropy first variation evaluated at $p_T^{pre}$.
\begin{tcolorbox}[colframe=white!, top=2pt,left=2pt,right=2pt,bottom=2pt]

\begin{center}
\textbf{Regularized Entropy First Variation Maximization}
\begin{equation}
    \argmax_{\pi}\quad  \langle \delta \entropy \left( p^{pre}_T \right), p_T^\pi \rangle - \alpha D_{KL}(p_T^{\pi}, p_T^{pre})\label{eq:entropy_first_variation}
\end{equation}
\end{center}
\end{tcolorbox}\vspace{-0mm}
\subsection{Generative exploration via density estimation}\vspace{-0mm}
\looseness -1 Crucially, this algorithmic principle does not rely on explicit uncertainty quantification and uses the generative model's ability to represent the density $p^{pre}_T$ to direct exploration. 
By introducing a function $f: \X \to \R$ defined for all $x \in \X$ as:\vspace{-0mm}
\begin{equation}
    f(x) \coloneqq \delta \entropy \left( p^{pre}_T \right) (x) = - \log \left( p^{pre}_T \right) (x) \label{eq:f_from_entropy_linearized} \vspace{-0.5mm}
\end{equation}
the exploration principle in Equation \eqref{eq:entropy_first_variation} computes a policy $\pi^*$ inducing $p^{\pi^*}_T$ with high density in regions where $p^{pre}_T$ has low density due to limited pre-training samples. Moreover, the KL regularization in Equation \eqref{eq:entropy_first_variation} implicitly enforces $p_T^{\pi^*}$ to lie on the approximate data manifold $\Omega_{pre}$. Formally, we have that:\vspace{-0mm}
\begin{equation}
    \Omega_{\pi^*} \coloneqq \supp(p_T^{\pi^*}) \subseteq \supp(p^{pre}_T) = \Omega_{pre} \quad \forall \alpha>0 \vspace{-0mm}
\end{equation}

\looseness -1 The entropy first variation in Equation \eqref{eq:f_from_entropy_linearized} can be interpreted as a measure of \emph{surprise}, while the entropy functional in Equation \eqref{eq:opt_problem} as expected surprise \citep{achiam2017surprise}.
\subsection{Easy to optimize, but hard to estimate density}\vspace{-0mm}
Existing fine-tuning methods for diffusion models can only optimize linear functionals of $p_T^\pi$, namely $\mathcal{L}(\mu) = \langle f, \mu \rangle \in \mF(\X)$, since they can be represented as classic (reward) functions $f: \X \to \R$, defined over the design space $\X$, \eg space of molecules. Although the entropy functional $\entropy$ in Equation \eqref{eq:opt_problem} is non-linear with respect to $p_T^\pi$, its first variation is a linear functional. As a consequence, by rewriting it as shown in Equation \eqref{eq:f_from_entropy_linearized}, it can be optimized using existing fine-tuning methods for classic reward functions via stochastic optimal control schemes~\citep[\eg][]{uehara2024feedback, domingo2024adjoint, zhao2024scores}, where the fine-tuning objective is:\vspace{-0mm}
\begin{equation*}
    \pi^* \in \argmax_\pi \EV_{x \sim \pi}\bigg[- \log \left( p^{pre}_T \right)(x)\bigg] - \alpha D_{KL}(p_T^{\pi}, p_T^{pre})\vspace{-0mm}
\end{equation*}

\looseness -1 We have shown that exploration can be self-guided by a generative model using its representational power of the density it induces. But unfortunately, estimating this quantity (\ie $p_T^{pre}$) is well-known to be a  challenging task in real-world high-dimensional settings~\citep{song2020score,kingma2021variational, skreta2024superposition}.\vspace{-0mm}

\subsection{Generative exploration without density estimation} \vspace{-0mm}
In the following, we show that in the case of diffusion models, the entropy's first variation at the marginal density $p_T^\pi$ induced by $\pi$, as in Equation \eqref{eq:f_from_entropy_linearized} with $\pi = \pi^{pre}$, can be optimized fully bypassing density estimation. 
This can be achieved by leveraging the following fundamental connection between the gradient of the entropy first variation $\nabla_x \delta \entropy \left( p^\pi_T \right)$ and the score $s^{\pi}(\cdot, T)$.

\begin{tcolorbox}[colframe=white!, top=2pt,left=2pt,right=2pt,bottom=2pt]

\begin{center}
\textbf{Gradient of entropy first variation = Negative score}
\begin{equation}
    \nabla_x \delta \entropy \left( p^\pi_T \right) = - \nabla_x \log p_T^\pi \simeq - s^{\pi}(\cdot, T) \label{eq:gradient_as_score}
\end{equation}
\end{center}
\end{tcolorbox}

Using Equation \eqref{eq:gradient_as_score}, it is possible to solve the maximization problem in Equation \eqref{eq:entropy_first_variation} by leveraging a first-order fine-tuning method such as Adjoint Matching \cite{domingo2024adjoint} with $\nabla_x f(x) \coloneqq - s^{pre}(x,T)$ as reward gradient, where $s^{pre}$ is the known (neural) score model approximating the true score function. This realization overcomes the limitation of density estimation and renders the method scalable for high-dimensional real-world problems. For the sake of completeness, we report a detailed pseudocode of its implementation in Appendix \ref{sec:implementation}. \vspace{-0mm}

\subsection{Beyond maximum entropy exploration}\vspace{-0mm}
\looseness -1 As shown in Section \ref{sec:experiments}, beyond the goal of maximum (entropy) exploration, the principle in Equation \eqref{eq:entropy_first_variation} can be used to achieve the desired trade-off between exploration and validity by controlling the regularization coefficient $\alpha$. Higher $\alpha$ values lead to a fine-tuned model $\pi$ that conservatively aligns with the validity encoded in the pre-trained model $\pi^{pre}$. In contrast, low $\alpha$ values enable exploration of low density regions within the approximate data manifold $\Omega_{pre}$, as illustrated in Figure \ref{fig:surprise_maximization}. The latter modality is particularly relevant when a validity checker is available, \eg synthetic accessibility (SA) scores for molecules~\citep{ertl2009estimation}, or formal verifiers for logic circuits~\citep{coudert1990unified}, allowing the discovery of new valid designs that expand the current manifold or dataset, effectively performing a guided data augmentation process~\citep{zheng2023toward}. 

\looseness -1 In particular, one might wonder if there exists a value of $\alpha$ such that the obtained fine-tuned model can provably solve the maximum entropy exploration problem in Equation \eqref{eq:opt_problem}. In the following section, we present a theoretical framework that answers this question positively under the idealized assumptions of exact score estimation and optimization oracle.

\section{Provably Optimal Exploration in One Step} 
\label{sec:theory1}
In this section, we show that under the assumptions of exact optimization and estimation of the entropy first variation $\delta \entropy \left( p^{pre}_T \right)$, a single fine-tuning step using Equation \eqref{eq:entropy_first_variation} yields an optimally explorative policy $\pi$ for entropy maximization over $\Omega_{pre}$. Complete proofs are reported in Appendices \ref{sec:app-theory1} and \ref{sec:app-theory2}.

We start by recalling the notion of Bregman divergence induced by a functional $\Q \in \mF(\X)$ between two densities $\mu, \nu \in \mP(\X)$, namely:
\begin{equation*}
    D_\Q(\mu, \nu) \coloneqq \Q(\mu) - \Q(\nu) - \langle \delta \Q(\nu), \mu - \nu \rangle
\end{equation*}
\looseness -1 Next, we introduce two structural properties for our analysis.\footnote{In line with standard notations in the optimization literature, we present the framework as
$
\min_{p_T^\pi \in \mathbb{P}(\Omega_{\mathrm{pre}})} - \mathcal{H}(p_T^\pi)$, which is clearly equivalent to \eqref{eq:opt_problem}.
}
\begin{restatable}[Relative smoothness and relative strong convexity \citep{lu2018relatively}]{definition}{relativeProperties}
\label{definition:relative_properties}
Let $\F: \mP(\X) \to \R$ a convex functional. We say that $\F$ is L-smooth relative to $\Q \in \mF(\X)$ over $\mP(\X)$ if $\exists$ $L$ scalar s.t. for all $\mu, \nu \in \mP(\X)$:
\begin{equation}
    \F(\nu) \leq \F(\mu) + \langle \delta \F(\mu), \nu - \mu \rangle + LD_\Q(\nu, \mu)
\end{equation}
and we say that $\F$ is l-strongly convex relative to $\Q \in \mF(\X)$ over $\mP(\X)$ if $\exists$ $l \geq 0$ scalar s.t. for all $\mu, \nu \in \mP(\X)$:
\begin{equation}
    \F(\nu) \geq \F(\mu) + \langle \delta \F(\mu), \nu - \mu \rangle + lD_\Q(\nu, \mu)
\end{equation}
\end{restatable}

\looseness -1 In the following, we view the principle in Equation \eqref{eq:entropy_first_variation} as a step of mirror descent~\citep{nemirovskij1983problem} and the KL divergence term as the Bregman divergence induced by an entropic mirror map $\Q = -\entropy$, \ie $D_{KL}(\mu, \nu) = D_{-\entropy}(\mu, \nu)$.
We can now state the following lemma regarding $\F = -\entropy$.
\begin{restatable}[Relative smoothness and strong convexity for $\F=\Q=\entropy$]{lemma}{entropyRelProperties}
\label{lemma:entropy_rel_properties}
For $\F=\Q=-\entropy$ as in Equation \eqref{eq:entropy_first_variation}, we have that $\F$ is 1-smooth (\ie $L=1$) and 1-strongly convex (\ie $l=1$) relative to $\Q$.
\end{restatable}

We can finally state the following set of idealized assumptions as well as the one-step convergence guarantee.

\begin{restatable}[Exact estimation and optimization]{assumption}{allAssumptions}
\label{assumption:all_assumptions}
We consider the following assumptions:\vspace{-0mm}
\begin{enumerate}
    \item Exact score estimation: $s^{pre}(\cdot, T) = \nabla_x \log p_T^{pre}$\vspace{-0mm}
    \item The optimization problem in Equation \eqref{eq:entropy_first_variation} is solved exactly.\vspace{-0mm}
\end{enumerate}
\end{restatable}
\begin{tcolorbox}[colframe=white!, top=2pt,left=2pt,right=2pt,bottom=2pt]
\begin{restatable}[One-step convergence]{theorem}{exactDerivativeConvergence}
\label{theorem:exact_derivative_convergence}
Given Assumptions \ref{assumption:all_assumptions}, fine-tuning a pre-trained model $\pi^{pre}$ according to Equation \eqref{eq:entropy_first_variation} with $\alpha = L = 1$, leads to a policy $\pi$ inducing a marginal distribution $p^{\pi}_T \in \mP(\Omega_{pre})$ such that:
\begin{equation}
    \entropy(p_T^*) - \entropy(p^{\pi}_T) \leq \frac{L-l}{K}D_{KL}(p_T^*, p_T^{pre}) = 0
\end{equation}
where $p_T^* \coloneqq p_T^{\pi^*}$ is the marginal distribution induced by the optimal exploratory policy $\pi^* \in \argmax_{\pi \in \Lambda} \entropy(p^{\pi}_T)$ with $\Lambda = \{\pi : p^{\pi}_T \in \mP(\Omega_{pre})\}$ being the set of policies compatible with the approximate data manifold $\Omega_{pre}$.
\end{restatable}
\end{tcolorbox}

Theorem \ref{theorem:exact_derivative_convergence} hints at promising performances of using Equation \eqref{eq:entropy_first_variation} as a fine-tuning objective for maximum entropy exploration, as under Assumptions \ref{assumption:all_assumptions}, it provably leads to an optimally explorative policy in one step. However, Assumptions \ref{assumption:all_assumptions} clearly do not hold in most realistic settings as:\vspace{-0mm}
\begin{enumerate}
    \item The estimation quality of the score $s^{pre}(\cdot, T)$ learned from data is approximate.\vspace{-0mm}
    \item The high-dimensional stochastic optimal control methods used to solve Equation \eqref{eq:entropy_first_variation} are approximate in practice.\vspace{-0mm}
\end{enumerate}

\looseness -1 Therefore, optimizing the entropy first variation as in Equation \eqref{eq:entropy_first_variation} is actually unlikely to lead to an optimally explorative diffusion model, as shown experimentally in Section \ref{sec:experiments}. To address this issue, in the next section we propose an exploration algorithm by building on the exploration principle in Equation \eqref{eq:entropy_first_variation}.

\section{Algorithm: \AlgNameLong} 
\label{sec:algorithm}
\looseness -1 In the following, we present \AlgNameDef, see Algorithm \ref{alg:memd_algorithm}, which reduces manifold exploration to sequential fine-tuning of the pre-trained diffusion model $\pi^{pre}$ by following a mirror descent (MD) scheme~\citep{nemirovskij1983problem}. Crucially, each iteration $k$ of \AlgNameShort corresponds to a fine-tuning step according to Equation  \eqref{eq:entropy_first_variation}, where the pre-trained model $\pi^{pre} \eqqcolon \pi_0$ is then replaced by the model at the previous iteration, namely $\pi_{k-1}$. Concretely, this makes it possible to reduce the optimization of the entropy functional, which is non-linear \wrt $p_T^{\pi}$, to a sequence of optimization problems of linear functionals.

\begin{algorithm}[H]
    \caption{\AlgNameDef}
    \label{alg:memd_algorithm}
        \begin{algorithmic}[1]
        \INPUT{ $K: $ number of iterations, $\pi^{pre}: $ pre-trained diffusion, $\{\alpha_k \}_{k=1}^{K}$ regularization coefficients}
        \STATE{\textbf{Init:} $\pi_0 \coloneqq \pi^{pre} $}
        \FOR{$k=1, 2, \hdots, K$}
        \STATE{Set: $\nabla_x f_{k} = - s^{k-1}$ with $s^{k-1} = s^{\pi_{k-1}}$}
        \STATE{Compute $\pi_k$ via first-order linear fine-tuning:
        \vspace{-4pt}
            \begin{equation*}
                \pi_k \leftarrow \LinearFineTuningSolver(\nabla_x f_k, \alpha_k, \pi_{k-1}) \vspace{-4pt}
            \end{equation*}
        }
        \ENDFOR
        \OUTPUT policy $\pi \coloneqq \pi_{K}$
        \end{algorithmic}
\end{algorithm}\vspace{-6pt}
Algorithm \ref{alg:memd_algorithm} requires as inputs a pre-trained diffusion model $\pi^{pre}$, the number of iterations $K$, and a schedule of regularization coefficients $\{\alpha_k\}_{k=1}^K$. At each iteration, \AlgNameShort sets the gradient of the entropy first variation evaluated at the previous policy $\pi_{k-1}$, namely $\nabla_x \delta \entropy \left( p^{k-1}_T \right)$, to be the score $s^{k-1} \coloneqq s^{\pi_{k-1}}$ associated to the diffusion model $\pi_{k-1}$ obtained at the previous iteration (line 3). Then, it computes policy $\pi_k$ by solving the following fine-tuning problem
\begin{equation*}
% \label{eq:oracle}
    \argmax_{\pi}\quad  \EV_{x \sim \pi}\bigg[- \log \left( p^{k-1}_T \right)(x)\bigg]  - \alpha_k KL(p_T^{\pi}, p_T^{k-1})
\end{equation*}
\looseness -1 via a first-order solver such as Adjoint Matching \citep{domingo2024adjoint}, using $\nabla f_k \coloneqq - s^{k-1}(\cdot,T)$ as in Eq. \eqref{eq:gradient_as_score} (line 4). Ultimately, it returns a final policy $\pi \coloneqq \pi_K$.
We report a possible implementation of \LinearFineTuningSolver in Appendix \ref{sec:implementation}.

Crucially, \AlgNameShort controls the distributional behavior of the final diffusion model $\pi$, which is essential to optimize the entropy as it is a non-linear functional over $\mP(\Omega_{pre})$. 

\looseness -1 However, it is still unclear whether the algorithm provably converges to the optimally explorative diffusion model $\pi^*$. In the next section, we answer affirmatively this question by developing a theoretical analysis under general assumptions based on recent results for mirror flows~\citep{hsieh2019finding}.

\section{Manifold Exploration Guarantees} 
\label{sec:theory2}
The purpose of this section is to establish a realistic framework under which \cref{alg:memd_algorithm} is guaranteed to solve the maximum entropy manifold exploration Problem \eqref{eq:opt_problem}.\vspace{-4pt}
\subsection{Key Assumptions}

We now present all the assumptions and provide an explanation of why they are realistic. Conceptually, these assumptions align with the \emph{stochastic approximation} framework of \citet{benaim2006dynamics,mertikopoulos2024unified, hsieh2021limits}. Specifically, recall that $p_T^k \coloneqq p^{\pi_k}_T$ represents the (stochastic) density produced by the \LinearFineTuningSolver oracle at the \(k\)-th step of \AlgNameShort, and consider the following \emph{mirror descent} iterates:\vspace{-2mm}
\begin{equation}
\label{eq:MD}
\tag{MD$_k$}
    \pMD^{k} \coloneqq \argmax_{p \in \mathbb{P}(\Omega_{pre})} \quad \langle \der \entropy \left(  p_T^{\pi_{k-1}} \right), p \rangle - \frac{1}{\step_k} D_{KL}(p,  p_T^{\pi_{k-1}})
\end{equation}  
where $1 / \gamma_k = \alpha_k$ in Algorithm \ref{alg:memd_algorithm}.  

\looseness -1 As explained in \cref{sec:theory1}, the maximum entropy manifold exploration problem \eqref{eq:opt_problem} can be solved in a single step using \eqref{eq:MD}. However, in realistic settings where only noisy \emph{and} biased approximations of \eqref{eq:MD} are available, it becomes essential to quantify the deviations due to these approximations from the idealized iterates in \eqref{eq:MD}. Additionally, the step sizes \(\step_k\) must be carefully designed to account for such deviations. This section aims to achieve precisely this goal. To this end, we first require:
\begin{assumption}[Support Compatibility]
\label{asm:support}
We assume that $\text{supp}(p_T^{\pi_k}) \subset \tilde{\Omega} \text{ for all } k, $ and $ supp(p_j^{\pi_k}) = \tilde{\Omega} $ for some $ j$.
%$\supp(p^{\pi_k}_T) \subseteq \supp(p_{pre})$ for all $k$.
\end{assumption}\vspace{-2mm}
Next, we require a purely technical assumption that is typically satisfied in practice:
\begin{assumption}
\label{asm:precompact}
    The sequence $\{\delta \entropy(p_T^{\pi_k})\}_k$ is precompact in the topology induced by the $L_\infty$ norm.\vspace{-2mm}
\end{assumption}
\looseness -1 Now, denote by $\hist_k$ the filtration up to step $k$, and consider the decomposition of the oracle into its \emph{noise} and \emph{bias} parts: 
\begin{align}
    \bias_k &\coloneqq \EV \left[ \delta \entropy(p^{\pi_k}_T) -   \delta \entropy(\pMD^k)  \,  \vert\, \hist_k \right]\\
    \noise_k &\coloneqq\delta \entropy(p^{\pi_k}_T) -   \delta \entropy(\pMD^k)  -  \bias_k \vspace{-7mm}
\end{align}

Observe that, conditioned on \(\hist_k\), \(\noise_k\) has zero mean, while \(\bias_k\) captures the \emph{systematic} error. We then impose the following:% for a positive constant \(\uupper\), it holds that:
\begin{assumption}[Noise and Bias]
\label{asm:approximate}
The following events happen almost surely:
\begin{align}
    \label{eq:bias-asym}
    &\norm{\bias_k}_\infty \rightarrow 0 \\  
\label{eq:uupper}
    &\sum_k \EV \left[ \step_k^2 \left( \norm{\bias_k}_\infty^2 + \norm{\noise_k}_\infty^2 \right) \right] <\infty\\
    \label{eq:bias-step}
    &\sum_k{\step_k\norm{\bias_k}_\infty} < \infty
\end{align}
\end{assumption}\vspace{-4mm}
\looseness -1 Two important remarks are worth noting. First, \eqref{eq:bias-asym} represents a \emph{necessary} condition for convergence: if this condition is violated, it becomes straightforward to construct examples where no practical algorithm can successfully solve the maximum entropy problem. Second, \eqref{eq:uupper} and \eqref{eq:bias-step} address the trade-off between the \emph{accuracy} of the approximate oracle \LinearFineTuningSolver and the aggressiveness of the step sizes, $\step_k$. Intuitively, a smaller noise and bias allows for the use of larger step sizes. In this context, \eqref{eq:uupper} and \eqref{eq:bias-step} establish a concrete criterion ensuring that the task of finding the optimally explorative policy succeeds with probability 1.

We are now finally ready to state the following result.
\begin{tcolorbox}[colframe=white!, top=2pt,left=2pt,right=2pt,bottom=2pt]
\begin{restatable}[Convergence under general assumptions]{theorem}{main}
\label{thm:main}
Consider the standard Robbins-Monro step-size rule: $\sum_k \step_k = \infty, \sum_k \step_k^2 < \infty$. Then under \crefrange{asm:support}{asm:approximate}, the sequence of marginal densities $p^k_T$ induced by the iterates $\pi_k$ of \cref{alg:memd_algorithm} converges weakly to $p_T^*$ almost surely. Formally, we have that:
\begin{equation}
     p_T^k \rightharpoonup p_T^* \quad \text{a.s.}
\end{equation}
where $p_T^* \in \argmax_{p_T \in \mP(\Omega_{pre})} \entropy(p_T)$ is the maximum entropy marginal density compatible with $\Omega_{pre}$.
\end{restatable}
\end{tcolorbox}
\textbf{Remark.}~It is possible to derive an explicit convergence rate for \cref{thm:main} corresponding to $\tilde{\mathcal{O}}((\log\log k)^{-1})$, which is, in general, the best achievable without additional assumptions~\cite{karimi2024sinkhorn}. We omit the technical details, as they are rather involved and offer limited practical relevance.

\section{Experimental Evaluation}
\label{sec:experiments}

\newlength{\imgw}
\setlength{\imgw}{0.25\textwidth}

\begin{figure*}[htbp]
    \centering
    \begin{subfigure}{\imgw}
        \centering
        \includegraphics[width=\textwidth]{images/toyIntro.pdf}
        \caption{True data support}
        \label{fig:toy_ex_a}
    \end{subfigure}%
    \begin{subfigure}{\imgw}
        \centering
        \includegraphics[width=\textwidth]{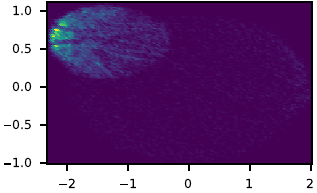}
        \caption{Pre-trained model sample}
        \label{fig:toy_ex_b}
    \end{subfigure}%
    \begin{subfigure}{\imgw}
        \centering
        \includegraphics[width=\textwidth]{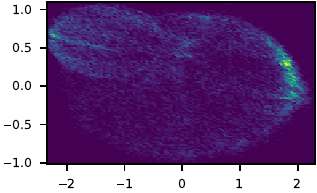}
        \caption{Fine-tuned model sample}
        \label{fig:toy_ex_c}
    \end{subfigure}%
    \begin{subfigure}{\imgw}
        \centering
        \includegraphics[width=\textwidth]{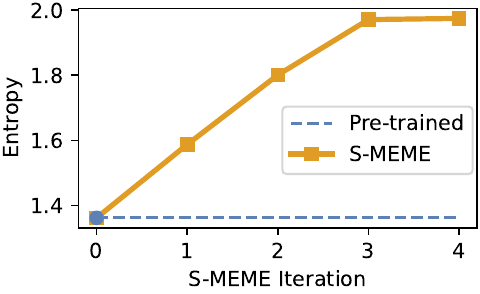}
        \caption{Entropy evaluation}
        \label{fig:toy_ex_d}
    \end{subfigure}\vspace{-0pt}
    \caption{\looseness -1 Illustrative example with unbalanced pre-trained model $\pi^{pre}$. (\ref{fig:toy_ex_a}) Support of true data distribution composed of a small high-density region (yellow) and a wide low-density area (green). (\ref{fig:toy_ex_b}) Sample from pre-trained model $\pi^{pre}$. (\ref{fig:toy_ex_c}) Sample from $\pi_4$ obtained after $4$ steps of \AlgNameShort. (\ref{fig:toy_ex_d}) Entropy estimation of densities $\{p_k\}_{k=1}^K$ obtained via $K=4$ steps. Notice that \AlgNameShort returns a fine-tuned model with significantly higher entropy than $\pi^{pre}$ (see \ref{fig:toy_ex_c} and \ref{fig:toy_ex_d}), and higher density in low-density regions for the pre-trained model (compare (\ref{fig:toy_ex_a}) and (\ref{fig:toy_ex_c})), while preserving the data support shown in \ref{fig:toy_ex_a}.}
    \label{fig:toy_ex} \vspace{-0pt}
\end{figure*}

\newlength{\imw}
\setlength{\imw}{0.12\textwidth}
\begin{figure*}[!t]
    \centering
    \includegraphics[width=1\textwidth]{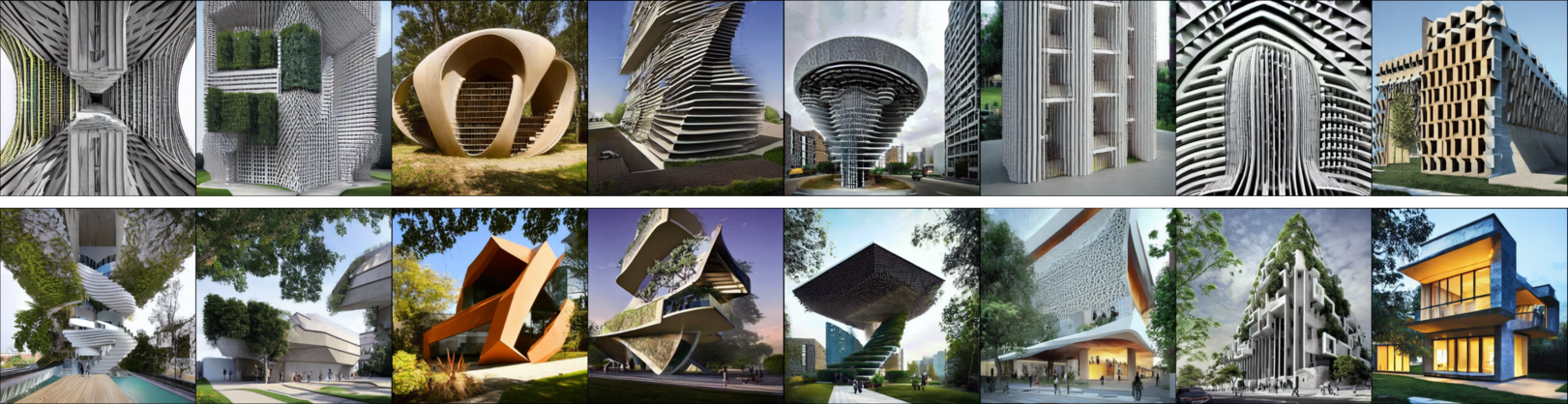}\vspace{-0pt}
    \caption{\looseness -1 Generated images from $\pi^{pre}$ (top)  and $\pi_3$ (bottom) for a fixed set of initial noisy samples using the prompt "A creative architecture.". We observe an increase in complexity and originality of the \AlgNameShort generated images while preserving semantic faithfulness, likely hinting at higher probability of sampling from a lower-density region of $\pi^{pre}$.}
    \label{fig:images_architecture}\vspace{-0pt}
\end{figure*}

In this section, we analyze the ability of \AlgNameShort to induce explorative policies on two tasks: (1) An illustrative example to showcase visually interpretable exploration and ability to sample from low-density regions (see Figure \ref{fig:toy_ex}), and (2) A text-to-image task aiming to explore the approximate manifold of \emph{creative architecture} designs (see Figure \ref{fig:images_architecture}).
Additional details on experiments are provided in \cref{sec:experiment-details}.

\mypar{(1) Illustrative setting}~~ \looseness -1 In this experiment, we consider the common setting where the density $p_T^{pre}$ induced by a pre-trained model $\pi^{pre}$ presents a high-density region (yellow area in Figure \ref{fig:toy_ex_a}) and a low-density region (green area in in Figure \ref{fig:toy_ex_a}). As illustrated in Figure \ref{fig:toy_ex_a}, the pre-trained model $\pi^{pre}$ induces an unbalanced density, where $N=80000$ samples are obtained mostly from the high-density area. 
For quantitative evaluation, we compute a Monte Carlo estimate of $\entropy(p_T^\pi)$.
Crucially, Figure \ref{fig:toy_ex} shows that \AlgNameShort can induce a highly explorative density in terms of entropy (see Figure \ref{fig:toy_ex_d}), compared with the pre-trained model, after only $K=4$ iterations. One can notice that the density induced by the fine-tuned model  (see Figure \ref{fig:toy_ex_c}) is significantly more uniform and higher in low-probability regions for the pre-trained model (see right region of Figure \ref{fig:toy_ex_c}), while preserving the support of the data distribution. 

\begin{table}[ht!]
\centering
\resizebox{\columnwidth}{!}{%
\begin{tabular}{lcccc}
\toprule
 & $\mathbf{p_T^{pre}}$ & \textbf{\AlgNameShort 1} & \textbf{\AlgNameShort 2} & \textbf{\AlgNameShort 3} \\
\midrule
\textbf{FID}$\mathbf{(p, p_T^{pre})}$  & 0.0 & 10.25 & 9.83  & 19.15  \\
\textbf{CLIP} & 22.27 & 20.79 &   20.88 &  19.86 \\
$\mathbf{\widehat \entropy(p, p_T^{pre})}$  & -1916.47 & 564.72 &  482.81 &  843.88 \\
\bottomrule
\end{tabular}
}
\caption{FID, CLIP and cross-entropy evaluation of $p_T^{pre}$ and $p_T^{\pi_k}$. For $k=1,2,3$, \AlgNameShort achieves larger distance to $p_T^{pre}$ while preserving high CLIP score.}\vspace{-4pt}
\label{tab:fid-clip}
\end{table}

\mypar{(2) Text-to-image manifold exploration}~~ \looseness -1 We consider the problem of exploring the data manifold of \emph{creative architecture} designs given a pre-trained text-to-image diffusion model. For this we utilize the stable diffusion (SD) 1.5 ~\citep{rombach2021highresolution} checkpoint pre-trained on the LAION-5B dataset~\citep{schuhmann2022laion}.
Since SD-1.5 uses classifier-free guidance~\citep{ho2022classifier}, we fine-tune the velocity resulting from applying the classifier-free guidance with a guidance scale of $w=8$ which is standard for SD-1.5.
Similarly, we also use the same guidance scheme for the fine-tuned model.
We fine-tuned the checkpoint with $K=3$ iterations of \AlgNameShort on a single Nvidia H100 GPU for the prompt "A creative architecture.". 
In \cref{fig:images_architecture}, we show images generated from  $\pi^{pre}$, $\pi_1$ and $\pi_3$, resulting from  the same initial noise samples.
One can notice an increase in the complexity and originality of the respective images, likely hinting at higher probability of the fine-tuned model to sample from a lower-density region for $\pi^{pre}$. Moreover, less conservative architectures are sampled with more steps of \AlgNameShort while preserving semantic faithfulness.
We measure this in \cref{tab:fid-clip} by computing the Fréchet inception distance (FID)~\citep{heusel2017gans} and Gaussian cross-entropy in feature space of Inception-v3  between $p_T^{\pi_k}$ and $p_T^{pre}$   for $k=1,2,3$ as well as the CLIP score~\citep{hessel2021clipscore} for the distribution induced by the specific prompt.
The main reason for these proxy metrics is the intractability of computing $\log p_T^{pre}(x)$ in high-dimensional spaces, such as that of images generated by a large diffusion model.
One can notice from \cref{tab:fid-clip} an increase in FID and cross-entropy between the distributions as $k$ increases, while the fine-tuned model preserves the CLIP score of the pre-trained model.
We provide further results for text-to-image in \cref{sec:extra_images}, and in Appendix \ref{app:diversity-measures}, we evaluate the Vendi score~\citep{friedman2022vendi} as a diversity metric.

\section{Related Work}
\label{sec:related_works}
In the following, we present relevant work in related areas.

\mypar{Maximum State Entropy Exploration}
\looseness -1 Maximum state entropy exploration, introduced by \citet{hazan2019maxent}, addresses the pure-exploration RL problem of maximizing the entropy of the state distribution induced by a policy over a dynamical system's state space~\citep[\eg][]{lee2019efficient, mutti2021task, guo2021geometric}. The presented manifold exploration problem is closely related, with $p^\pi_T$ representing the state distribution induced by policy $\pi$ over a subset of the state space. Nonetheless, in Problem \eqref{eq:opt_problem} the admissible state distributions $\mP(\Omega_{pre})$ are represented only implicitly via a pre-trained generative model $\pi^{pre}$, able to capture complex design spaces, \eg valid molecules. Moreover, exploration is guided by the diffusion model's score function via Eq. \ref{eq:gradient_as_score}, overcoming the need of explicit entropy or density estimation, a fundamental challenge in this area~\citep{liu2021behavior, seo2021state, mutti2021task}. 
\looseness -1 Recent studies have tackled maximum entropy exploration with finite sample budgets~\citep[\eg][]{mutti2022importance, mutti2022challenging, mutti2023convex, prajapat2023submodular, de2024global}. We believe several ideas presented in this work can extend to such settings. Ultimately, to the best of our knowledge, this is the first work providing a rigorous theoretical analysis of maximum state entropy exploration over continuous state spaces, albeit for a specific sub-case, as well as leveraging this formulation for fine-tuning of diffusion models.

\mypar{Continuous-time RL}
\looseness -1 Continuous-time RL extends stochastic optimal control \citep{fleming2012deterministic} to handle unknown rewards or dynamics \citep[\eg][]{doya2000reinforcement, wang2020reinforcement}. Problem \eqref{eq:opt_problem} represents the pure exploration case of continuous-time RL, were the goal is to compute a (purely) exploratory policy $\pi$ over a subset of the state space $\Omega_{pre} \subseteq \X$ implicitly defined by a pre-trained generative model $\pi^{pre}$. 
Moreover, Problem \eqref{eq:opt_problem} can be further motivated as a continuous-time RL reward learning setting~\citep[\eg][]{lindner2021information, mutny2023active, de2024geometric}, where an agent aims to learn an unknown homoscedastic reward function such as toxicity over a molecular space~\citep{yang2022exploring}. To our knowledge, this is the first work that tackles maximum entropy exploration in a continuous-time RL setting.

\mypar{Diffusion models fine-tuning via optimal control}
\looseness -1 Recent works have framed diffusion models fine-tuning with respect to a reward function $f: \X \to \R$ as an entropy-regularized stochastic optimal control problem~\citep[\eg][]{uehara2024fine, tang2024fine, uehara2024feedback,  domingo2024adjoint}. 
In this work, we introduce a scalable fine-tuning scheme, based on first-order solvers for classic rewards \citep[\eg][]{domingo2024adjoint}, that optimizes a broader class of functionals requiring information about the full density $p^\pi_T$, such as entropy and alternative exploration measures~\citep{de2024global, hazan2019maxent}. This paves the way to using diffusion models for optimization of distributional objectives, rather than simple scalar rewards. 
Beyond classic optimization, our framework is particularly relevant for Bayesian optimization, or bandit, problems~\citep[\eg][]{ uehara2024feedback}, where the reward function to be optimized over the manifold is unknown and therefore exploration is essential.

\mypar{Sample diversity in diffusion models generation}
\looseness -1 The lack of sample diversity in diffusion model generation is a key challenge tackled by various works ~\citep[\eg][]{corso2023particle, um2023don, kirchhof2024sparse, sadat2024cadsunleashingdiversitydiffusion, um2025self}. 
These methods complement ours by enabling diverse sampling from the fine-tuned explorative model obtained via \AlgNameShort.
% , which optimizes a distributional rather than particle behavior
While prior works focus on generating diverse samples from a fixed diffusion model, ours provides a framework for manifold exploration as policy optimization via reinforcement learning. This enables scalable and provable maximization of typical exploration measures in RL, such as state entropy \citep{hazan2019maxent}. Among related works, \citet{miao2024training} shares the closest intent, but lacks a formal setting with exploration guarantees, and the diffusion model's exploration process relies on computing an external metric for exploration, rather than being self-guided via its own score function as \AlgNameShort achieves via Eq. \eqref{eq:gradient_as_score}.

\mypar{Optimization over probability measures via mirror flows}
\looseness -1 Recently, there has been a growing interest in analyzing optimization problems over spaces of probability measures. Existing works have explored applications including GANs \cite{hsieh2019finding}, optimal transport \cite{aubin2022mirror, leger2021gradient, karimi2024sinkhorn}, and kernelized methods \cite{dvurechensky2024analysis}. However, to the best of our knowledge, the case of the entropy-based objective in Problem \ref{eq:opt_problem}, along with its associated relative smoothness analysis framework, has not been previously addressed. Moreover, prior approaches do not leverage a key aspect of our framework: the admissible space of probability measures, $\mP(\Omega_{pre})$, is represented only implicitly via a pre-trained generative model, which can approximate complex data manifolds learned from data, such as molecular spaces. This idea of optimization with implicit constraints captured by generative models is novel and absent in earlier work.

\section{Conclusion}
\label{sec:conclusions}
\looseness -1 This work tackles the fundamental challenge of leveraging the representational power of generative models for exploration. We first introduce a formal framework for exploration as entropy maximization over the approximate data manifold implicitly captured by a pre-trained diffusion model. Then, we present an algorithmic principle that guides exploration via density estimation, a challenging task in real-world settings. By exploiting a fundamental connection between entropy and a diffusion model’s score function, we overcome this problem and ensure scalability of the proposed principle for exploration. Building on this, we introduce \AlgNameShort, a sequential fine-tuning algorithm that provably solves the exploration problem, with convergence guarantees grounded in recent advances in mirror flows. Finally, we validate the proposed method on both a conceptual benchmark and a high-dimensional text-to-image task, demonstrating its practical relevance.

\section*{Acknowledgements}
This publication was made possible by the ETH AI Center doctoral fellowship to Riccardo De Santi, and postdoctoral fellowship to Marin Vlastelica. The project has received funding from the Swiss
National Science Foundation under NCCR Catalysis grant number 180544 and NCCR Automation grant agreement 51NF40 180545. 

\section*{Impact Statement}
In this work we provide a method and theoretical analysis of maximum-entropy exploration for diffusion generative models as well as an implemented algorithm.
To the best of our knowledge, this will enable further more applied research in this areas with exciting applications.

\bibliography{biblio}
\bibliographystyle{icml2025}

%%%%%%%%%%%%%%%%%%%%%%%%%%%%%%%%%%%%%%%%%%%%%%%%%%%%%%%%%%%%%%%%%%%%%%%%%%%%%%%
%%%%%%%%%%%%%%%%%%%%%%%%%%%%%%%%%%%%%%%%%%%%%%%%%%%%%%%%%%%%%%%%%%%%%%%%%%%%%%%
% APPENDIX
%%%%%%%%%%%%%%%%%%%%%%%%%%%%%%%%%%%%%%%%%%%%%%%%%%%%%%%%%%%%%%%%%%%%%%%%%%%%%%%
%%%%%%%%%%%%%%%%%%%%%%%%%%%%%%%%%%%%%%%%%%%%%%%%%%%%%%%%%%%%%%%%%%%%%%%%%%%%%%%
\newpage
\appendix
\onecolumn

\section{Proof for \Cref{proposition:omega_is_compact}}
\label{sec:compactness}
Recall the probability flow ODE in \citep[eq. (13)]{song2020score}, which is what we use to generate $p^{pre}_T$ (this is also a common practice in the literature). 
We know that the generative ODE corresponding to the backward SDE \cref{eq:generative_SDE} is written as (suppose $\eta = 1$ for simplicity)
\begin{equation*}
    \mathrm{d}X^\leftarrow_t = \underbrace{- f(X^\leftarrow_t, T-t) + \frac{1}{2}g^2(T-t) s^{pre}(X^\leftarrow_t, T-t)}_{:=v(X^\leftarrow_t, t)} \mathrm{d}t.
\end{equation*}
Here $f$ is defined in the forward process \cref{eq:forward_process}.
Clearly the velocity field $v(x, t)$ is Lipschitz \wrt $x$ due to the assumption that $s^{pre}$ is Lipschitz and the fact that $f$ is linear \wrt $x$.
Consequently, \emph{the flow map induced by the above ODE is Lipschitz}. 
As a result, since $X_0$ is sampled from a truncated Gaussian distribution which has a compact support, $\Omega^{pre} = \mathrm{supp}(p^{pre}_T)$ is also compact for any finite $T$.
\newpage

\section{Proof for \cref{sec:theory1}}
\label{sec:app-theory1}
\exactDerivativeConvergence*
\begin{proof}
    Towards proving this result, we interpret Eq. \eqref{eq:entropy_first_variation} as the first iteration of Algorithm \ref{alg:memd_algorithm}. Hence, to prove the statement, it is sufficient to show that Algorithm \ref{alg:memd_algorithm} after one iteration computes $\pi_1$ inducing density $p^{\pi_1}_T$ such that $\entropy(p_T^*)=\entropy(p_T^\pi)$. We prove this result by leveraging the properties of relative smoothness and relative strong convexity introduced in Sec. \ref{sec:theory1}. 

    The analysis is bases on a classic analysis for mirror descent via relative properties~\citep{lu2018relatively} First, we show the following, where for the sake of using a simple notation, we denote $p_T^{\pi_k}$ by $\mu_k$, and consider an arbitrary density $\mu \in \mP(\Omega_{pre})$.
    \begin{align}
        \entropy(\mu_k) &\leq \entropy(\mu_{k-1}) + \langle \delta \entropy(\mu_{k-1}), \mu_k - \mu_{k-1} \rangle + LD_\Q(\mu_k, \mu_{k-1})\\
        &\leq \entropy(\mu_{k-1}) + \langle \delta \entropy(\mu_{k-1}), \mu - \mu_{k-1} \rangle + LD_\Q(\mu, \mu_{k-1}) - LD_\Q(\mu, \mu_{k})
    \end{align}
    where in the first inequality we have used the $L$-smoothness of $\entropy$ relative to $\Q = \entropy$ as in Definition \ref{definition:relative_properties}, while in the last inequality we have used the three-point property of the Bregman divergence~\citep[Lemma 3.1]{lu2018relatively} with $\phi(\mu) = \frac{1}{L} \langle \delta \entropy(\mu_{k-1}), \mu - \mu_{k-1} \rangle$, $z = \mu_{k-1}$, and $z^+ = \mu_k$. Then, we can derive: 
    \begin{equation}
        \entropy(\mu_k) \leq \entropy(\mu) + (L - \mu)D_\Q(\mu, \mu_{k-1}) - LD_\Q(\mu, \mu_k)
    \end{equation}
    by using the $l$-strong convexity of $\entropy$ relative to $\Q = \entropy$ as in Definition \ref{definition:relative_properties}. By induction, using monotonicity of the iterates and non-negativity of the Bregman divergence as in~\citep{lu2018relatively}, one obtains:
    \begin{align}
        \sum_{k=1}^K \left(\frac{L}{L-l}\right)^k\left( \entropy(\mu_k) - \entropy(\mu) \right) \leq LD_\Q(\mu, \mu_0) - L\left( \frac{L}{L-l} \right)D_\Q(\mu, \mu_k) \leq L D_\Q(\mu, \mu_0)
    \end{align}
    Defining:
    \begin{equation}
        \frac{1}{C_k} = \sum_{k=1}^K \left( \frac{L}{L-l} \right)^k
    \end{equation}
    and rearrenging the terms leads to:
    \begin{equation}
        \entropy(\mu_k) - \entropy(\mu) \leq C_k L D_\Q(\mu, \mu_0) = \frac{\mu D_Q(\mu, \mu_0)}{\left( 1 + \frac{l}{L-l} \right)^l - 1} \label{eq:last_eq_proof}
    \end{equation}
    Given Eq. \ref{eq:last_eq_proof}, the convergence in the statement can be derived using Lemma \ref{lemma:entropy_rel_properties}, and the fact that $\left( 1 + \frac{l}{L-l} \right)^k \geq 1 + \frac{k \mu}{L- \mu}$. Ultimately, $p^{\pi}_T \in \mP(\Omega_{pre}) \forall \alpha > 0$ is trivially due to the fact that $\Omega_{pre}$ is the support of the right element of the Kullback–Leibler divergence in Eq. \ref{eq:entropy_first_variation}.
\end{proof}
\newpage

\section{Proof for \cref{sec:theory2}}
\label{sec:app-theory2}

%**********************************************************************
%***    MACROS: GENERAL
%**********************************************************************
\newcommand{\debug}[1]{#1}

\newcommand{\newmacro}[2]{\newcommand{#1}{{#2}}}		% for shorthand definitions
\newcommand{\newop}[2]{\DeclareMathOperator{#1}{{#2}}}		% for shorthand definitions

\newcommand{\dual}{h}
\newcommand{\run}{k}
\newcommand{\obj}{\entropy}
\newcommand{\state}{\dual}
\newcommand{\curr}[1][\state]{\debug{#1}^{\run}}		% for current value (X by default)
\renewcommand{\next}[1][\state]{\debug{#1}^{\run+1}}		% for current value (X by default)

\newcommand{\efftime}{\tau}
\newcommand{\apt}[2][]{\state^{#1}(#2)}	
\newcommand{\ctime}{t}	
\newcommand{\defeq}{\coloneqq}
\newcommand{\runalt}{n}
\newcommand{\start}{1}

%\newmacro{\step}{\gamma}		% for step-size
\newmacro{\temp}{\eta}		% for learning rate
\newmacro{\points}{\mathcal{Z}}		% for point set
\newmacro{\intpoints}{\points^{\circ}}		%for point set interior
\newmacro{\point}{\dual}		% for generic point
\newmacro{\pointalt}{\alt\point}		% for alternate point
\newcommand{\orbit}[2][]{\point_{#1}(#2)}		% for orbit (x by default)
\newcommand{\dotorbit}[2][]{\dot\point_{#1}(#2)}		% for diff. orbit (x by default)

\newcommand{\vbound}{V}
\newcommand{\lips}{L}
%----------------------------------------------------------------------
%% Time and dynamics
%----------------------------------------------------------------------
%\newmacro{\ctime}{t}		% for continuous time
\newmacro{\ctimealt}{s}		% for dummy continuous time
\newmacro{\cstart}{0}		% for continuous time start

\newmacro{\horizon}{T}		% for horizon

\newmacro{\vecfield}{V}		% for vector field

\newmacro{\signal}{V}		% for signal
\newmacro{\error}{W}		% for error
%%\newmacro{\noise}{U}		% for noise
%\newmacro{\bias}{b}		% for bias
\newmacro{\brown}{W}		% for Wiener process

%----------------------------------------------------------------------
%% Sequences and recursions
%----------------------------------------------------------------------
%\newmacro{\state}{Z}		% for main iterate
\newmacro{\dstate}{Y}		% for other iterate

\newcommand{\avg}[1][\state]{\bar{#1}}		% for averaging (X by default)
\newcommand{\new}[1][\point]{#1^{+}}		% for new iterate (x by default)

\newcommand{\init}[1][\state]{\debug{#1}_{\start}}		% for initial value (X by default)
\newcommand{\afterinit}[1][\state]{\debug{#1}_{\afterstart}}		% for second value (X by default)
\newcommand{\preiter}[1][\state]{\debug{#1}_{\runalt-1}}		% for iterated value (X by default)
\newcommand{\iter}[1][\state]{\debug{#1}_{\runalt}}		% for iterated value (X by default)
\newcommand{\afteriter}[1][\state]{\debug{#1}_{\runalt+1}}		% for iterated value (X by default)
\newcommand{\preprev}[1][\state]{\debug{#1}_{\run-2}}		% for previous value (X by default)
\newcommand{\prev}[1][\state]{\debug{#1}_{\run-1}}		% for previous value (X by default)
\newcommand{\prelead}[1][\state]{\debug{#1}_{\run-1}^{+}}		% for current value (X by default)
\newcommand{\lead}[1][\state]{\debug{#1}_{\run}^{+}}		% for current value (X by default)
\newcommand{\from}{\colon}		% for function definition

\newmacro{\flowmap}{\Theta}		% for (semi)flows
\newcommandtwoopt{\flow}[2][\ctime][\point]{\flowmap_{#1}(#2)}
%----------------------------------------------------------------------
%% Min-Max Problems
%----------------------------------------------------------------------
\newmacro{\minmax}{\Phi}		% for minmax objective

\newmacro{\minvar}{x}		% for minimization variable
\newmacro{\minvaralt}{\alt x}		% for alternate minvar
\newmacro{\minvars}{\mathcal{X}}		% for minvar space

\newmacro{\maxvar}{y}		% for maximization variable
\newmacro{\maxvaralt}{\alt y}		% for alternate maxvar
\newmacro{\maxvars}{\mathcal{Y}}		% for maxvar space

\newmacro{\minsol}{\sol[\minvar]}		% for minimization solution
\newmacro{\maxsol}{\sol[\maxvar]}		% for maximization solution

\newcommand{\sol}[1][\point]{#1^{\ast}}		% for solutions (x by default)
\newcommand{\sols}{\sol[\points]}		% for set of solutions
\newcommand{\sdev}{\sigma}

\newcommand{\as}{\debug{\textpar{a.s.}}\xspace}		% for almost surely
\newmacro{\set}{\mathcal{S}}		% for generic set

\newmacro{\open}{\mathcal{U}}		% for open sets
\newmacro{\closed}{\mathcal{C}}		% for closed sets
\newmacro{\cpt}{\mathcal{K}}		% for compact sets
\newmacro{\nhd}{\mathcal{U}}		% for neighborhoods

%----------------------------------------------------------------------
%% Document layout
%----------------------------------------------------------------------
\newcommand{\afterhead}{.}		
%%\newcommand{\putperiod}[1]{#1.}
%\usepackage{titlesec}
%\titleformat{\subsection}[runin]{\bfseries}{\thesubsection.}{1ex}{}[\afterhead]
%\titleformat{\subsubsection}[runin]{\scshape}{\thesubsubsection.}{1ex}{}[]
%%\titleformat{\section}[runin]{\bfseries}{\thesection.}{1ex}{\afterhead}
%\titlespacing{\subsection}{0pt}{\medskipamount}{1em}
\newcommand{\para}[1]{\paragraph{\textbf{#1\afterhead}}}

%----------------------------------------------------------------------
%% Text and formatting
%----------------------------------------------------------------------
\newcommand{\cf}{cf.\xspace}		% for consistency
\newcommand{\vs}{vs.\xspace}		% for consistency

\newacro{APT}{asymptotic pseudotrajectory}
\newacroplural{APT}[APTs]{asymptotic pseudotrajectories}
\newacro{GD}{gradient dynamics}
\newacro{GF}{gradient flow}
\newacro{ICT}{internally chain-transitive}
\newacro{MDS}{martingale difference sequence}
\newacro{NE}{Nash equilibrium}
\newacroplural{NE}[NE]{Nash equilibria}
\newacro{ODE}{ordinary differential equation}
\newacro{SA}{stochastic approximation}
\newacro{SFO}{stochastic first-order oracle}
\newacro{SG}{stochastic gradient}
\newacro{SP}{saddle-point}
\newacro{WAC}{weak asymptotic coercivity}

\newacro{AH}{Arrow\textendash Hurwicz}
\newacro{BDG}{Burkholder\textendash Davis\textendash Gundy}
\newacro{ConO}{consensus optimization}
\newacro{RM}{Robbins\textendash Monro}
\newacro{KW}{Kiefer\textendash Wolfowitz}
\newacro{GDA}{gradient descent/ascent}
\newacro{SGA}{symplectic gradient adjustment}
\newacro{SGD}{stochastic gradient descent}
\newacro{SGDA}{stochastic gradient descent/ascent}
\newacro{SPSA}{simultaneous perturbation stochastic approximation}
\newacro{ASGDA}[alt-SGDA]{alternating stochastic gradient descent/ascent}
\newacro{SEG}{stochastic extra-gradient}
\newacro{EG}{extra-gradient}
\newacro{PEG}{Popov's extra-gradient}
\newacro{RG}{reflected gradient}
\newacro{OG}{optimistic gradient}
\newacro{PPM}{proximal point method}

\newacro{GAN}{generative adversarial network}
\newacro{NN}{neural network}
\newacro{FTRL}{``follow the regularized leader''}
\newacro{CGD}{Competitive Gradient Descent}
\newacro{wp1}[w.p.$1$]{with probability $1$}

\newcommand{\ites}{\{\curr[\state]\}_{\run\in\mathbb{N}}}
%----------------------------------------------------------------------

\subsection{Proof of \cref{thm:main}}

We restate the theorem for reader's convenience:
\main*

\begin{proof}
\newcommand{\drm}{\mathrm{d}}
To enhance the readability of our proof, we begin by outlining the key steps.

\para{Proof Outline} 

The main idea is to analyze the convergence of the iterates $\{p_T^k\}_{k\in \mathbb{N}}$ generated by \cref{alg:memd_algorithm} by relating them to a corresponding \emph{continuous-time} dynamical system. Specifically, we define the initial dual variable as
%to a \emph{continuous-time} dynamics. To this end, let
$$\dual_0 = \delta \entropy (p_{pre}) = -\log p_{pre},$$
and consider the following system:
\begin{equation}\label{eq:MF}
\tag{MF}
    \begin{cases}
    \dot{\dual}_t = \delta \entropy(p_t) \\
    p_t = \delta (-\entropy)^\star(\dual_t)
    \end{cases} \equiv \quad
    \begin{cases}
    \dot{\dual}_t = - \log p_t \\
    p_t = \frac{ e^{\dual_t}}{\int_\Omega e^{\dual_t}}.
    \end{cases}
\end{equation}
Here, $ (-\entropy)^\star(\dual) \coloneqq \log\int_\Omega e^\dual$ is the Fenchel dual of the entropy function \cite{hsieh2019finding, hiriart2004fundamentals}. 

To bridge the gap between discrete and continuous-time analysis, we construct a continuous-time interpolation of the discrete iterates $\ites$. Let $(\curr \coloneqq \delta \entropy(p^{\run}_T))_{\run\in\mathbb{N}}$ be the sequence of the corresponding \emph{dual variables}. We introduce the notion of an ``effective time'' $\curr[\efftime]$, defined as: $$\curr[\efftime] \defeq \sum_{\runalt=\start}^{\run} \step_\runalt,$$
which represents the cumulative time elapsed up to the $\run$-th iteration of the discrete-time process $\curr[\state]$ using step-size $\step_\run$. Using $\curr[\efftime]$, we define the \emph{continuous-time interpolation} $\apt{\ctime}$ of $\curr$ as follows:
\begin{equation}
\tag{Int}
\label{eq:interpolation}
\apt{\ctime}
	\defeq \curr
		+ \frac{\ctime - \curr[\efftime]}{\next[\efftime] - \curr[\efftime]} (\next - \curr).
\end{equation}

Intuitively, the convergence of our algorithm follows if the following two conditions hold:

\newtheorem{informalassumption}{Informal Assumption}
\begin{informalassumption}[Closeness of discrete and continuous times]
\label{iasm:dis2cont}
The interpolated process \eqref{eq:interpolation} asymptotically approaches the continuous-time dynamics in \eqref{eq:MF} as $\run \to \infty$.
\end{informalassumption}

\begin{informalassumption}[Convergence of continuous-time dynamics]
\label{iasm:sol}
The trajectory of \eqref{eq:MF} converges to the \textbf{optimal solution} of the maximum entropy problem \eqref{eq:opt_problem}.
\end{informalassumption}

To formalize the above intuition, we leverage the \emph{stochastic approximation} framework of \citet{benaim2006dynamics,mertikopoulos2024unified,karimi2024sinkhorn}, outlined as follows.

First, to precisely state \cref{iasm:dis2cont}, we introduce a measure of ``closeness'' between continuous orbits. Let $\points$ denote the space of integrable functions on $\Omega$ (viewed as the dual space of probability measures; see \cite{halmos2013measure}), and define the \emph{flow} $\flowmap\from\R_{+}\times\points\to\points$ associated with \eqref{eq:MF}. That is, for an initial condition $\point_0 = \point \in \points$, the function $\flowmap$ describes the orbit of \eqref{eq:MF} at time $\ctime\in\R_{+}$. 

We then define the notion of ``asymptotic closeness'' as follows:

\begin{definition}
\label{def:APT}
We say that $\apt{\ctime}$ is an \acdef{APT} of \eqref{eq:MF} if, for all $\horizon > 0$, we have:
\begin{equation}
\label{eq:APT}
\lim_{\ctime\to\infty}
	\sup_{0 \leq \ctimealt \leq \horizon}
		\norm{\apt{\ctime + \ctimealt} - \flow[\ctimealt][\apt{\ctime}]}_\infty
	= 0.
\end{equation}
\end{definition}

This comparison criterion, introduced by \citet{benaim1996asymptotic}, plays a central role in our analysis. Intuitively, it states that $\apt{\ctime}$ eventually tracks the flow of \eqref{eq:MF} with arbitrary accuracy over arbitrarily long time windows. Consequently, if \eqref{eq:interpolation} is an \ac{APT} of \eqref{eq:MF}, we can reasonably expect its behavior— and thus that of $\{\curr[\state]\}_{\run\in\mathbb{N}}$— to closely follow \eqref{eq:MF}. 

The precise connection is established by \citet{benaim1996asymptotic} through the concept of \acdef{ICT} sets:

\begin{definition}[\citealp{benaim1996asymptotic, benaim2006dynamics}]
\label{def:ICT}
Let $\set$ be a nonempty compact subset of $\points$. Then:
\begin{enumerate}
\item
$\set$ is \emph{invariant} if $\flow[\ctime][\set] = \set$ for all $\ctime\in\R$.
\item \label{item:b}
$\set$ is \emph{attracting} if it is invariant and there exists a compact neighborhood $\cpt$ of $\set$ such that $\lim_{\ctime\to\infty} \dist(\flow,\set) = 0$ uniformly for all $\point\in\cpt$.
\item
$\set$ is an \textbf{\acdef{ICT}} set if it is invariant and $\flowmap\vert_{\set}$ admits no proper attractors within $\set$.
\end{enumerate}
\end{definition}

The significance of \ac{ICT} sets lies in \citep[Theorem~5.7]{benaim2006dynamics}:

\begin{theorem}[\acp{APT} converge to \ac{ICT} sets]
\label{thm:apt2ict}
Let $\apt{\ctime}$ be a precompact \acl{APT} generated by $\{\curr[\state]\}_{\run\in\mathbb{N}}$ for the flow associated with the continuous-time system \eqref{eq:MF}. Then, almost surely, $\curr[\state] \to \set$, where $\set$ is an \ac{ICT} set of \eqref{eq:MF}.
\end{theorem}

By \cref{thm:apt2ict}, establishing \cref{thm:main} reduces to proving the following two statements:
\begin{enumerate}
\item The iterates $\ites$ of \cref{alg:memd_algorithm} generate a precompact \ac{APT} of \eqref{eq:MF}.
\item The unique \ac{ICT} set of \eqref{eq:MF} is the solution to the optimization problem \eqref{eq:opt_problem}.
\end{enumerate}

These results provide the rigorous counterpart to \crefrange{iasm:dis2cont}{iasm:sol}. The proof below proceeds by formally establishing each of these points.

\para{The \ac{ICT} set of \eqref{eq:MF} is the solution to \eqref{eq:opt_problem}}
By the definition \eqref{eq:MF}, we can easily see that:
\begin{align}
    \dot{p}_t &= p_t \dot{\dual}_t - \frac{e^{\dual_t}}{\int_\Omega e^{\dual_t}} \cdot \frac{\int_\Omega\dot{\dual}_t\cdot e^{\dual_t} }{\int_\Omega e^{\dual_t}} \\
    \label{eq:hold}
    &= p_t\left( \dot{\dual}_t - \mathbb{E}_{p_t} \dot{\dual}_t \right).
\end{align}
We then compute:
\begin{align}
   - \frac{\mathrm{d}}{\mathrm{d}t}\entropy(p_t) &= - \langle\delta \entropy(p_t), \dot{p}_t\rangle \\
    &= \langle  \log p_t,  p_t\left( \dot{\dual}_t - \mathbb{E}_{p_t} \dot{\dual}_t \right)\rangle \quad \quad \textup{by \eqref{eq:hold}} \\
    &= \int_\Omega p_t \log p_t \dot{\dual}_t - \int_\Omega p_t \log p_t \cdot \int_\Omega p_t \dot{\dual}_t \\
    &= - \int_\Omega p_t  (\log p_t)^2 -  \left( \int_\Omega p_t \log p_t \right)^2 \quad \quad \textup{by \eqref{eq:MF}} \\
    & = - \left(   \mathbb{E}_{X_t\sim p_t}(\log p_t(X_t) )^2  -  \left( \mathbb{E}_{X_t\sim p_t} \log p_t(X_t) \right)^2  \right) \\
    &\leq 0
    \label{eq:hold1}
\end{align}
by Jensen's inequality. Also, note that the inequality is strict if $\dual_t$ is not constant, i.e., if $p_t$ is not uniform on $\Omega$. 

In short, we established in \eqref{eq:hold1} that $\entropy(\cdot)$ serves as a \emph{Lyapunov function} for the continuous-time system \eqref{eq:MF}. Since $\entropy(\cdot)$ is \emph{strictly} concave, the only \ac{ICT} set is the singleton $\{p_T^*\}$, where $p_T^*$ represents the uniform (and hence entropy-maximizing) measure on $\Omega$ \citep[Proposition.~6.4]{benaim2006dynamics}.

\para{\cref{alg:memd_algorithm} generates an \ac{APT}}

Let $(p^{\run}_T)_{\run\in\mathbb{N}}$ be the sequence of measures on $\Omega$ generated by \cref{alg:memd_algorithm} with the oracle \LinearFineTuningSolver, and recall that its dual variables are given by $(\curr \coloneqq \delta \entropy(p^{\pi_\run}_T))_{\run\in\mathbb{N}}$. Also, recall the corresponding continuous-time interpolation \eqref{eq:interpolation}.

\Cref{asm:support} ensures that each dual variable $\curr$ is a well-defined function on $\Omega$ after some iteration $j$, while \cref{asm:precompact} guarantees the precompactness of $\apt{\cdot}$. Furthermore, under \cref{asm:approximate}, standard arguments (see, e.g., \textbf{Proposition 4.1} of \cite{benaim2006dynamics} or \cite{karimi2024sinkhorn}) establish that $\apt{\cdot}$ generates an \ac{APT} of the continuous-time flow defined by \eqref{eq:MF}. Finally, \cref{thm:apt2ict} ensures that $\ites$ converges almost surely to an \ac{ICT} set of \eqref{eq:MF}, which we have already shown to contain only $\{p_T^*\}$.

Therefore, applying the theory of \cite{hsieh2021limits, karimi2024sinkhorn}, we conclude that, almost surely,
\begin{align}
\lim_{\run \to \infty} \dual^\run = \lim_{\run \to \infty} \delta \obj(p^{\run}_T) 
=\lim_{\run \to \infty} -\log  p^{\run}_T 
= \delta \obj(p_T^*) \quad\quad \textup{in } L_\infty.
\label{eq:hold3}
\end{align}
Since $\Omega$ is compact, \eqref{eq:hold3} implies that, for any smooth test function $\psi$ on $\Omega$, $\langle p^{\run}_T, \psi \rangle \to \langle p_T^*, \psi \rangle  $, which completes the proof.
\end{proof}
\newpage

\section{Detailed Example of Algorithm Implementation}
\label{sec:implementation}
\subsection{Pseudocode for implementation of Eq. \eqref{eq:entropy_first_variation}}
\looseness -1 For the sake of completeness, in the following we present the pseudocode for a possible implementation of a \LinearFineTuningSolver via a first-order optimization method, used to solve \eqref{eq:entropy_first_variation}, as well as within \AlgNameShort. In particular, we present the same implementation we use in Sec. \ref{sec:experiments}, based on Adjoint Matching \citep{domingo2024adjoint}, which captures the linear fine-tuning  via a stochastic optimal control problem and solves it via regression.

In the following, we adopt the notation from the Adjoint Matching paper \citep[Apx E.4]{domingo2024adjoint}. We denote the pre-trained noise predictor by $\epsilon^{pre}$, the fine-tuned one as $\epsilon^{\mathrm{finetuned}}$, and with $\Bar{\alpha}$ the cumulative noise schedule, as used by \citet{ho2020denoising}. The complete algorithm is presented in Algorithm \ref{alg:step_implemented}. First, notice that given a noise predictor $\epsilon$ (as Defined in Sec. \ref{sec:background}) and a cumulative noise schedule $\bar{\alpha}$, one can define the score $s$ as follows \citep{song2019generative}:  
\begin{equation}
    s(x,t) \coloneqq -\frac{\epsilon(x,t)}{\sqrt{1-\bar{\alpha}_t}}
\end{equation}

\begin{algorithm}[H]
    \caption{\LinearFineTuningSolver (Implementation based on Adjoint Matching \citep{domingo2024adjoint})}
    \label{alg:step_implemented}
        \begin{algorithmic}[1]
        \INPUT{ $N: $ number of iterations, $\epsilon^{pre}: $ pre-trained noise predictor, $\alpha$ regularization coefficient, $m:$ trajectories batch size, $\nabla f$: reward function gradient}
        \STATE{\textbf{Init:} $\epsilon^{\mathrm{finetuned}} \coloneqq \epsilon^{pre}$ with parameter $\theta$}
        \FOR{$n=0, 2, \hdots, N-1$}
            \STATE{Sample $m$ trajectories $\{X_t\}_{t=1}^T$ according to DDPM \citep{song2020score}}, \eg sample $\epsilon_t \sim \mathcal{N}(0,I), \; X_0 \sim \mathcal{N}(0,I)$
            \[
            X_{t+1}
            = 
            \sqrt{\frac{\bar{\alpha}_{t+1}}{\bar{\alpha}_t}}
            \left(
              X_t
              - \frac{
                  1 - \frac{\bar{\alpha}_t}{\bar{\alpha}_{t+1}}
                }{
                  \sqrt{1 - \bar{\alpha}_t}
                }
                \,\epsilon^{\mathrm{finetuned}}(X_t,t)
            \right)
            + 
            \sqrt{\frac{1 - \bar{\alpha}_{t+1}}{1 - \bar{\alpha}_t}
            \left(
              1 - \frac{\bar{\alpha}_t}{\bar{\alpha}_{t+1}}
            \right)}
            \epsilon_t
            \]
        Use reward gradient: $$\Tilde{a}_T = \nabla f(X_T)$$
        For each trajectory, solve the lean adjoint ODE, see \citep[Eq. 38-39]{domingo2024adjoint}, from $t=T$ to $0$:
        \[
        \bar{a}_{k}
        = 
        \bar{a}_{t+1}
        + 
        \bar{a}_{t+1}^\top 
        \nabla_{X_t}
        \left(
          \sqrt{\frac{\bar{\alpha}_{t+1}}{\bar{\alpha}_{t}}}
          \Bigl(
            X_t 
            \;-\;
            \frac{
              1 - \frac{\bar{\alpha}_t}{\bar{\alpha}_{t+1}}
            }{
              \sqrt{\,1 - \bar{\alpha}_t\,}
            }
            \,\epsilon^{\mathrm{pre}}(X_t,t)
          \Bigr)
          \;-\;
          X_t
        \right)
        \]
        Where $X_t$ and $\Tilde{a}_t$ are computed without gradients, \ie $X_t = \texttt{stopgrad}(X_t), \Tilde{a}_t = \texttt{stopgrad}(\Tilde{a}_t)$.
        For each trajectory compute the Adjoint Matching objective \citep[Eq. 37]{domingo2024adjoint}:
        \[
        \mathcal{L}(\theta)
        = \sum_{t=0}^{T-1}
        \biggl\|
          \sqrt{\frac{\bar{\alpha}_{t+1}}{\bar{\alpha}_t\,\bigl(1-\bar{\alpha}_{t+1}\bigr)}}
          \Bigl(1 - \frac{\bar{\alpha}_t}{\bar{\alpha}_{t+1}}\Bigr)\!
          \Bigl(\epsilon^{\mathrm{finetuned}}(X_t,t)\;-\;\epsilon^{\mathrm{pre}}(X_t,t)\Bigr)
          \;-\;
          \sqrt{\frac{1-\bar{\alpha}_{t+1}}{1-\bar{\alpha}_t}}
          \Bigl(1 - \frac{\bar{\alpha}_t}{\bar{\alpha}_{t+1}}\Bigr)
          \,\bar{a}_t
        \biggr\|^2
        \]
        Compute the gradient $\nabla_\theta \mathcal{L}(\theta)$ and update $\theta$.
        \ENDFOR
        \OUTPUT Fine-tuned noise predictor $\epsilon^{\mathrm{finetuned}}$
        \end{algorithmic}
\end{algorithm}
\newpage

\section{Experiment Details}
\label{sec:experiment-details}
In this section we provide further details on the experiments.

\paragraph{Illustrative setting.} Pre-training was performed by standard denoising score-matching and uniform samples, namely $10K$, from the two distributions in Fig. \ref{fig:toy_ex_a}. For fine-tuning, in this experiment we ran \AlgNameShort for $6000$ gradient steps in total, for $K=1,2,3,4$.
Notably, k=$1$ amounts to having a fixed reward for fine-tuning, $-\log p_T^{pre}(x)$.
Each round of \AlgNameShort performs $6000/K$ for the particular experiment.
In this way we observe the effect of having more rounds of the mirror descent scheme, with same number of gradient updates.
Since we utilize Adjoint Matching~\citep{domingo2024adjoint} for the linear solver in \cref{alg:memd_algorithm}, we perform an iteration of \cref{alg:step_implemented} by first sampling $20$ trajectories via DDPM of length $400$ that are used  for solving the lean adjoint ODE with the reward $-\lambda \nabla \log p_T(x)$ and $\lambda=0.1$.
Subsequently we perform $2$ stochastic gradient steps by the Adam optimizer with batch size $2048$, initialized with learning rate $4 \times 10^{-4}$.
For the density plots in \cref{fig:toy_ex} we sampled $80000$ points with $100$ DDPM steps.
To obtain \cref{fig:toy_ex_d}, we computed a Monte-Carlo estimate of $\entropy(p_T^{\pi_k})$ with an approximation of $\log p_T^{\pi_k}(x)$ resulting from the instantaneous change of variables and divergence flow equation,
\begin{equation}
    \log p(x) = \log p_0(x) + \int_0^T \nabla \cdot f(x_t, t) dt,
\end{equation}
where $f$ is the velocity of the probability-flow ODE for the variance-preserving forward process of the diffusion.

\paragraph{Text-to-image architecture design.} For obtaining \cref{fig:images_architecture}, similarly as in the illustrative example we used \cref{alg:step_implemented} as the linear solver.
At each iteration of \cref{alg:step_implemented} we sample $4$ trajectories of length $60$ by DDPM, conditioned on the prompt on top of which we perform $10$ Adam steps with initial learning rate $3 \times 10^{-7}$ and batch size $8$.
Each iteration of \cref{alg:memd_algorithm} entails $20$ iterations of \cref{alg:step_implemented}.
We ran \cref{alg:memd_algorithm} for $K=3$.
For this experiment, we used $\lambda=0.1$.

\paragraph{Text-to-image evaluation.} Evaluating the entropy of $p_T^{\pi_k}$ is computationally prohibitive for the case of the high-dimensional latent of SD-1.5.
Consequently, we opted for proxy metrics to quantify how much does the distribution change with increase of $\pi_k$, the FID score for distributional distance and CLIP score for semantic alignment.
In addition, we computed the cross-entropy in \cref{tab:fid-clip} between the Gaussians in the Inception-v3 feature space, where the Gaussians were fitted the same way as in computing the FID score.
The FID score, cross-entropy and CLIP score have been computed on $3000$ samples from respective conditional distributions.

\section{Diversity Measures}
\label{app:diversity-measures}

Picking a proper diversity measure for the space of images is non-trivial.
For completeness, we provide here additional results for the Vendi score~\citep{friedman2022vendi}, which is a kernel diversity metric.
This however again comes with a particular choice of kernel and feature map $\phi$.
We picked the standard RBF kernel,
\begin{equation}
   \textbf{k}(x,x') = \exp(-\gamma\|\phi(x) - \phi(x') \|),
\end{equation}
where we set $\gamma = 50$ and pick the Inception-v3 model feature space.
Given computation constraints, we re-ran another experiment with the prompt "A creative architecture." and evaluated the sampled for each MD iteration with $300$ images (samples are given in \cref{fig:images_vendi_run}) with $\lambda=0.1$ and initial learning rate for Adam being $10^{-6}$ with $2$ MD steps and $100$ gradient steps per MD iteration.
The reported numbers for Vendi and CLIP scores in \cref{tab:vendi-result} are computed over $3$ seeds.

\begin{table}[ht!]
\centering
\begin{tabular}{lccc}
\toprule
 & $\mathbf{p_T^{pre}}$ & \textbf{\AlgNameShort 1} & \textbf{\AlgNameShort 2} \\
\midrule
\textbf{Vendi}  & $1.53$ & $1.7\pm 0.08$ & $1.63 \pm 0.07$   \\
\textbf{CLIP} & $0.22$  & $0.22 \pm 0.01$ & $0.22\pm 0.01$  \\
\bottomrule
\end{tabular}

\caption{FID, CLIP and cross-entropy evaluation of $p_T^{pre}$ and $p_T^{\pi_k}$. For $k=1,2$, \AlgNameShort achieves larger distance to $p_T^{pre}$ while preserving high CLIP score.}\vspace{-4pt}
\label{tab:vendi-result}
\end{table}

\newlength{\imww}
\setlength{\imww}{0.12\textwidth}
\begin{figure*}[!t]
    \centering
    \includegraphics[width=1\textwidth]{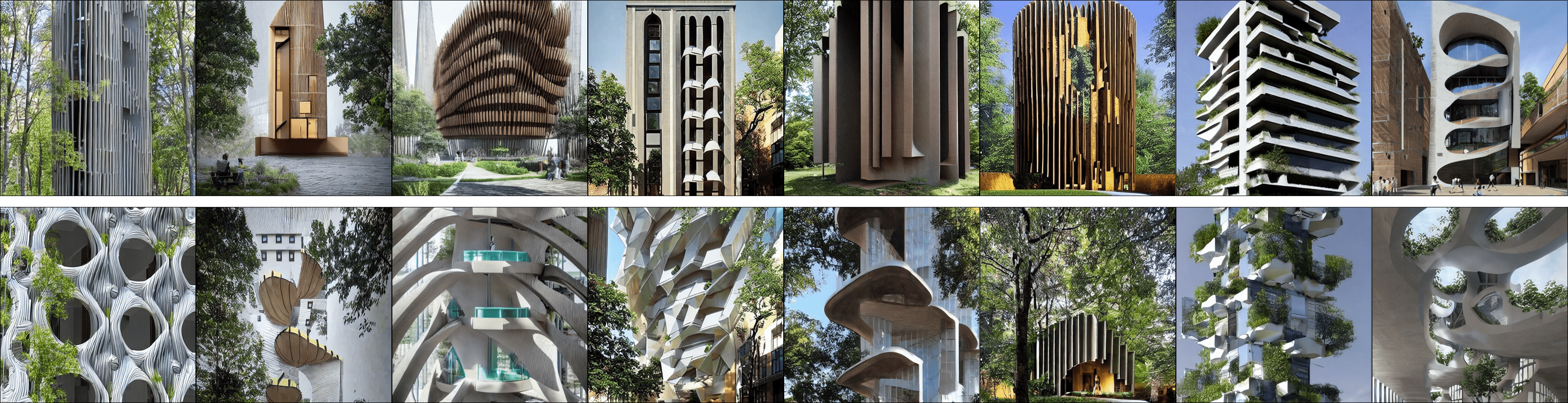}\vspace{-0pt}
    \caption{\looseness -1 Generated images from $\pi^{pre}$ (top)  and $\pi_3$ (bottom) for a fixed set of initial noisy samples using the prompt "A creative architecture.".}
    \label{fig:images_vendi_run}\vspace{-0pt}
\end{figure*}

\section{Additional Text-to-Image Results}
\label{sec:extra_images}
In the following, we present additional experimental results obtained via the same text-to-image pre-trained diffusion model introduced in Sec. \ref{sec:experiments}, and with experimental details as presented within Sec. \ref{sec:experiment-details}.  
\newpage
\begin{figure*}
    \centering
    \includegraphics[width=0.95\textwidth]{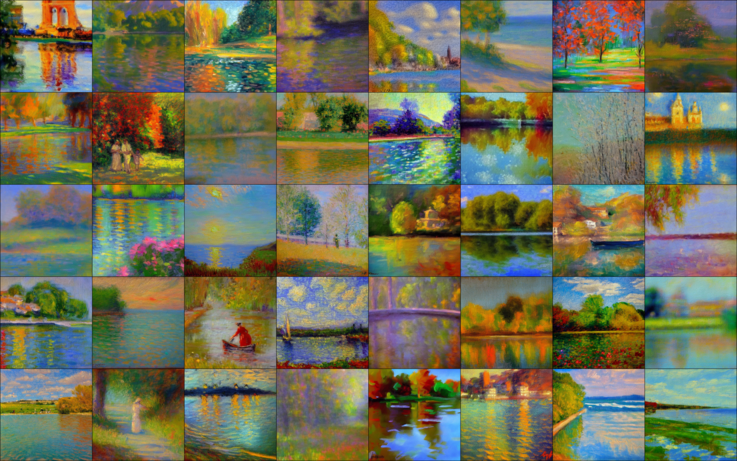}
    \caption{\looseness -1 Generated images from $\pi^{pre}$ with prompt "A creative impressionist painting."}
    \label{fig:images_paintings}
\end{figure*}

\begin{figure*}
    \centering
    \includegraphics[width=0.95\textwidth]{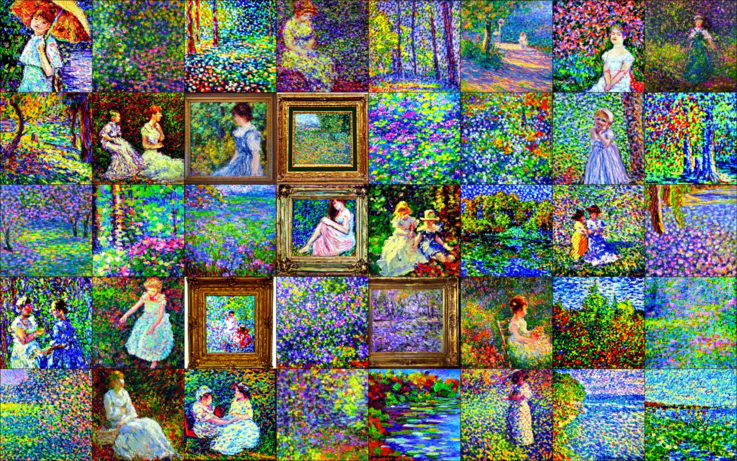}
    \caption{\looseness -1 Generated images obtained via fine-tuning of $\pi^{pre}$ via \AlgNameShort with prompt "A creative impressionist painting."}
    \label{fig:images_painting_fine}
\end{figure*}

\begin{figure*}
    \centering
    \includegraphics[width=0.95\textwidth]{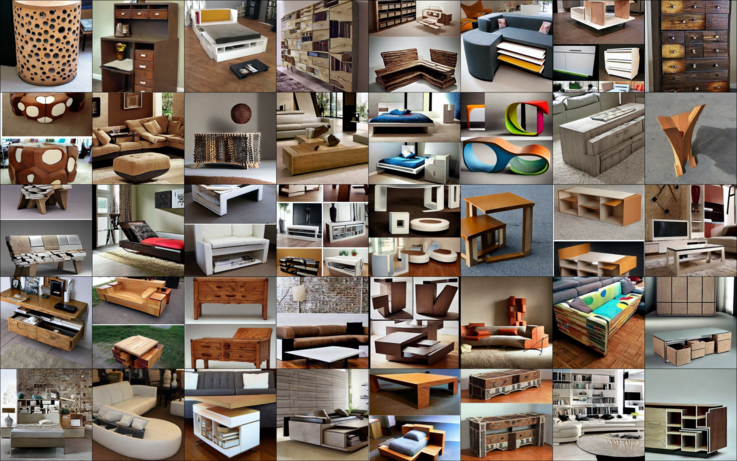}
    \caption{\looseness -1 Generated images from $\pi^{pre}$ with prompt "Creative furniture."}
    \label{fig:images_furniture}
\end{figure*}

\begin{figure*}
    \centering
    \includegraphics[width=0.95\textwidth]{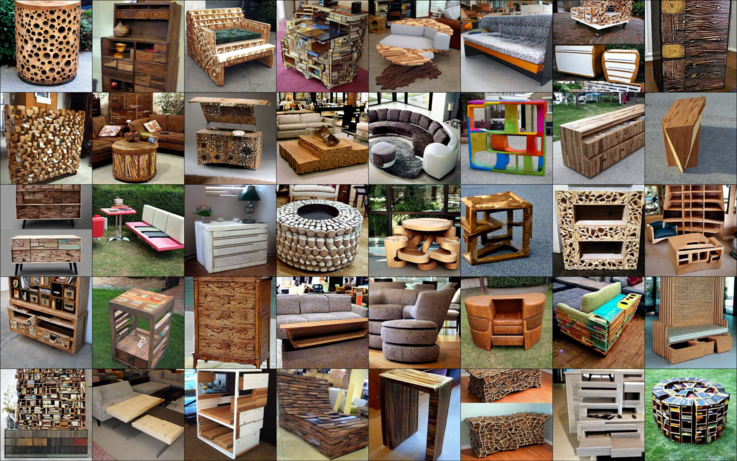}
    \caption{\looseness -1 Generated images obtained via fine-tuning of $\pi^{pre}$ via \AlgNameShort with prompt "Creative furniture."}
    \label{fig:images_furniture_fine}
\end{figure*}

\begin{figure*}
    \centering
    \includegraphics[width=0.95\textwidth]{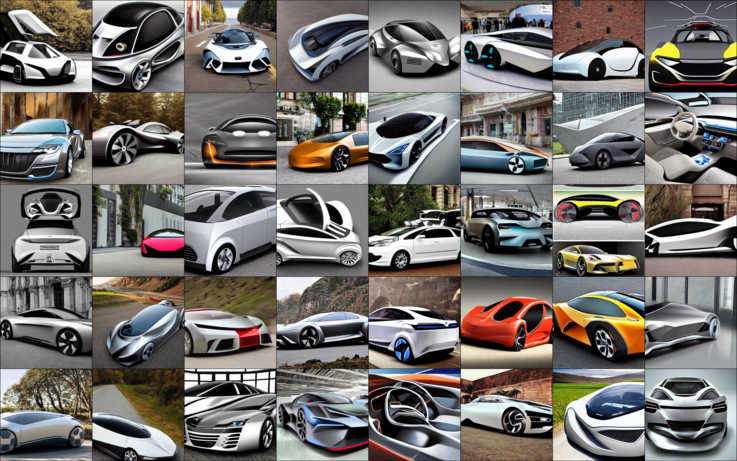}
    \caption{\looseness -1 Generated images from $\pi^{pre}$ with prompt "An innovative car design."}
    \label{fig:images_car}
\end{figure*}

\begin{figure*}
    \centering
    \includegraphics[width=0.95\textwidth]{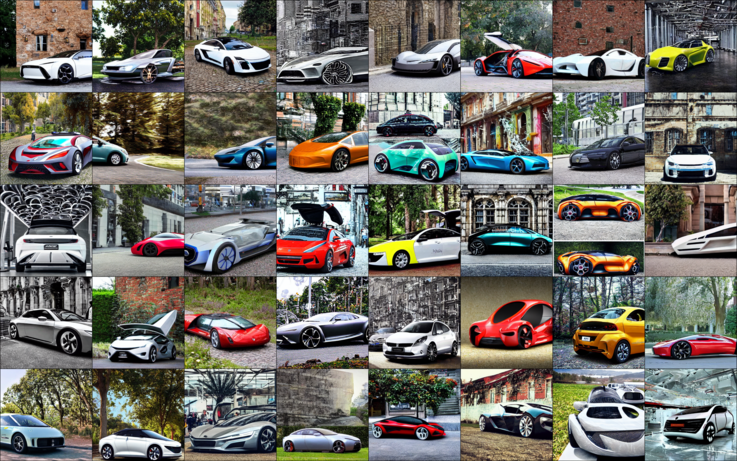}
    \caption{\looseness -1 Generated images obtained via fine-tuning of $\pi^{pre}$ via \AlgNameShort with prompt "An innovative car design."}
    \label{fig:images_car_fine}
\end{figure*}
\newpage

%%%%%%%%%%%%%%%%%%%%%%%%%%%%%%%%%%%%%%%%%%%%%%%%%%%%%%%%%%%%%%%%%%%%%%%%%%%%%%%
%%%%%%%%%%%%%%%%%%%%%%%%%%%%%%%%%%%%%%%%%%%%%%%%%%%%%%%%%%%%%%%%%%%%%%%%%%%%%%%

\end{document}